%% file: bare_jrnl.tex
\documentclass[lettersize,journal]{IEEEtran}
\usepackage{amsmath,amsfonts}
\usepackage{algorithmic}
\usepackage{algorithm}
\usepackage{array}
\usepackage{caption}
\usepackage{subcaption}
\usepackage{multirow}

\usepackage{thmtools}
\usepackage{thm-restate}

\usepackage{hyperref}

\usepackage{cleveref}

\usepackage{tabularx}
\usepackage[detect-weight=true]{siunitx}
\usepackage{colortbl}
\usepackage{booktabs}
\usepackage{textcomp}
\usepackage{stfloats}
\usepackage{url}
\usepackage{verbatim}
\usepackage{graphicx}
\usepackage[table,xcdraw]{xcolor}
\usepackage{color,soul}
\usepackage{cite}
\usepackage{etoolbox}

\soulregister\cite7
\soulregister\ref7
\soulregister\eqref7
\soulregister\pageref7
\sethlcolor{white}
\makeatletter
\def\SOUL@hlpreamble{%
   \setul{\dp\strutbox}{\dimexpr\ht\strutbox+\dp\strutbox\relax}%
   \let\SOUL@stcolor\SOUL@hlcolor
   \SOUL@stpreamble
}

\AtBeginEnvironment{quote}{\vspace{1em}}
\AtEndEnvironment{quote}{\vspace{1em}}

\input{bdhmath}

\definecolor{orange-left}{rgb}{0.85, 0.325, 0.098}
\definecolor{blue-right}{rgb}{0, 0.447, 0.741}
\definecolor{alizarin}{rgb}{0.82, 0.1, 0.26}

\definecolor{A}{rgb}{0.6350, 0.0780, 0.1840}
\definecolor{B}{rgb}{0.5, 0, 0.5}
\definecolor{C}{rgb}{0.3010, 0.7450, 0.9330}
\definecolor{D}{rgb}{0.8500, 0.3250, 0.0980}
\definecolor{E}{rgb}{0.4660, 0.6740, 0.1880}
\definecolor{F}{rgb}{0.9290, 0.6940, 0.1250}
\definecolor{G}{rgb}{0, 0.4470, 0.7410}
\definecolor{H}{rgb}{0, 0, 0}

\definecolor{R4}{rgb}{0.0, 0.4, 0.65}
\definecolor{R6}{rgb}{1.0, 0.22, 0.0}
\definecolor{R7}{rgb}{0.8, 0.0, 0.8}
\definecolor{R8}{rgb}{0.4, 0.22, 0.33}
\definecolor{R9}{rgb}{0.65, 0.04, 0.37}

\newcommand{\Ai}{{\color{A}{\textbf{WIMF}}}}
\newcommand{\B}{{\color{B}{\textbf{WIKSVD}}}}
\newcommand{\C}{{\color{C}{\textbf{PCA}}}}
\newcommand{\E}{{\color{E}{\textbf{DMD}}}}
\newcommand{\Dc}{{\color{D}{\textbf{ICA}}}}
\newcommand{\F}{{\color{F}{\textbf{EMD}}}}
\newcommand{\G}{{\color{G}{\textbf{PCA}}}}

\def\I{\mathbf{I}}
\def\A{\mathbf{A}}
\def\d{\mathbf{d}}
\def\x{\mathbf{x}}
\def\Z{\mathbf{Z}}

\def\Re{\mathbb{R}}
\def\D{\mathbf{D}}
\def\X{\mathbf{X}}
\def\L{\mathbf{L}}
\def\Y{\mathbf{Y}}
\def\k{\mathbf{k}}

\begin{document}

\title{Wave Physics-Informed Matrix Factorizations}

\author{Harsha Vardhan Tetali, Joel B. Harley, Benjamin D. Haeffele
        }

\markboth{To appear in IEEE Transactions on Signal Processing}%
{Shell \MakeLowercase{\textit{et al.}}: A Sample Article Using IEEEtran.cls for IEEE Journals}


\IEEEoverridecommandlockouts
\IEEEpubid{\makebox[\columnwidth]{~\copyright~2023 IEEE, DOI: 10.1109/TSP.2023.3348948 \hfill} \hspace{\columnsep}\makebox[\columnwidth]{ }}

\maketitle

\IEEEpubidadjcol

\begin{abstract}
 With the recent success of representation learning methods, which includes deep learning as a special case, there has been considerable interest in developing techniques that  incorporate known physical constraints into the learned representation. As one example, in many applications that involve a signal propagating through physical media (e.g., optics, acoustics, fluid dynamics, etc), it is known that the dynamics of the signal must satisfy constraints imposed by the wave equation. Here we propose a matrix factorization technique that decomposes such signals into a sum of components, where each component is regularized to ensure that it {nearly} satisfies wave equation constraints. Although our proposed formulation is non-convex, we prove that our model can be efficiently solved to global optimality. Through this line of work we establish theoretical connections between wave-informed learning and filtering theory in signal processing. We further demonstrate the application of this work on modal analysis problems commonly arising in structural diagnostics and prognostics. 
\end{abstract}


%
\IEEEpeerreviewmaketitle

\section{Introduction}
%
%
%
%



\IEEEPARstart{R}{epresentation} learning has gained importance in fields that utilize data generated through physical processes, such as weather forecasting \cite{mehrkanoon2019deep}, manufacturing \cite{ringsquandl2017knowledge}, structural health monitoring \cite{zobeiry2020theory}, acoustics \cite{mohamed2012understanding}, and medical imaging \cite{litjens2017survey}. Moreover, learning physically consistent and interpretable features can improve our understanding of the physically viable information about the data and the composition of the system or process generating it. Unfortunately, however, features of the data learned through generic machine learning algorithms typically do not correspond to physically interpretable quantities. 

For example, physically consistent solutions are important to spatio-temporal modal analysis, where recovering modes of vibration yield important information about a physical system or process (e.g., buildings, bridges, space shuttles, satellites, acoustic instruments)  \cite{stanbridge1999modal, YANG2017232, articleKammer, articleCunha, Marshall1985ModalAO}. Furthermore, these problems are also relevant to frequency-wavenumber spectral analysis \cite{1449208}, which is the basis for many advances in RADAR \cite{1449208}, seismic \cite{10.1093/gji/ggy286}, and acoustic \cite{YU201643} beamforming and signal processing, where each sinusoid may correspond to a different source or target. 
Physically consistent solutions provide a trustworthy basis in terms of the physics of the system in contrast to any other basis, which might be obtained by generic machine learning algorithms. This suggests the need for learning algorithms which choose bases that are physically consistent.

\subsection{Related Work}

Modified learning paradigms, suited to different physical application domains, have begun to draw interest \cite{karpatne2017theory}. For example, several researchers have designed physics-informed neural networks to learn approximate solutions to a partial differential equation \cite{BenjaminErichson2019,Raissi2019,Nabian2018,Tartakovsky2018,Long2018}. These have been recently applied to ultrasonic surface waves to extract velocity parameters and denoise data \cite{Shukla2020}. Similar physics-guided neural networks \cite{karpatne2017theory,Karpatne2017a,Lei2018,Mestav2018} use physics-based regularization to improve data processing, demonstrating that regression tasks (such as estimating the temperature throughout a lake) can be more accurate and robust when compared with purely physics-based solutions or purely data-driven solutions. These approaches demonstrate the strong potential for physics-informed machine learning but remain early in their study. 

Representation theory has been applied to other wave-based problems. For example, generative models have been studied for applications in electroencephalography \cite{Luo2020} and seismology \cite{Mosser2020}, although these learning systems do not assume physical knowledge. Other prior work has considered sparse generative representation of ultrasonic waves through the use of sparse signal processing \cite{harley2013sparse}, dictionary learning \cite{alguri2018baseline,alguri2017model,alguri2021sim}, and neural network sparse autoencoders \cite{Alguri2019}. These methods utilize the fact that many waves can be represented as a small, sparse sum of spatial modes and attempt to extract these modes within the data \cite{alguri2018baseline}. These dictionary learning algorithms have further been combined with wave-informed regularizers to incorporate physical knowledge of wave propagation into the solution \cite{Tetali2019, tetali2021wave}.


One of the key challenges {for many of these problems} is that they typically require solving a non-convex optimization problem, {and as a result the learned representation can be highly dependent on details such as initialization or choice of optimization algorithm}. Attention has been devoted recently to the non-convex optimization problems that arise in representation learning, and positive results have begun to emerge in certain settings.  For example, within the context of low-rank matrix recovery, it has been shown that gradient descent is guaranteed to converge to a local minimum and all local minima will be globally optimal under certain circumstances \cite{bhojanapalli2016global, ge2016matrix, park2016non, ge2017no}.
Further, other recent work has studied `structured' matrix factorization problems, where one promotes properties in the matrix factors by imposing regularization on the factorized matrices (e.g., if one desires a factor to be sparse an $\ell_1$ norm penalty could be imposed on one of the matrix factors). It has been shown that such problems can also be solved to global optimality in certain circumstances, but whether such guarantees can be made depends critically on how the regularization is formulated \cite{haeffele2014structured,haeffele2019structured,bach2013convex}.

{\subsection{Contributions} \label{subsection:contributions}
{In this work we make the following contributions. We \emph{first} introduce a \textit{\textbf {wave-informed matrix factorization model}}, which ensures that the recovered components minimize wave equation constraints, allowing for direct interpretability of the factors based on assumed physics.} This is distinguished from traditional factorization methods, such as {principal component analysis (PCA)}, which are not constrained by physical knowledge as well as from {models which explicitly parameterize wave-like factors}, and are often too constrained for a problem, particularly in the presence of noise. In other words, our approach balances the utilization of both prior assumptions from physics and fills the gap between purely data driven methods \cite{doi:10.1177/14759217221075241} and purely physics based methods.

\emph{Second},} although our resulting formulation requires solving a \emph{non-convex} optimization problem, {we identify its form amenable to be provably solved to global optimality} \cite{haeffele2019structured}. {This differs from prior wave-physics informed dictionary learning }\cite{Tetali2019}, {which provides no such guarantees.}

{\emph{Third}, unlike physics-informed neural networks, we show that our wave physics-informed matrix factorization has a satisfying interpretation as an optimal filter problem. Specifically, we show that} our  optimization algorithm {based on} \cite{haeffele2019structured} { reveals a solution is that of an optimal bank of 1st-order Butterworth filters. This interpretability enables confidence in the algorithm and intelligent selection of  hyper-parameters.}

{In summary, this work's novel contributions are in demonstrating the first framework for matrix factorization that is constrained solely on the wave equation. Unlike wave-informed dictionary learning }\cite{Tetali2019}{, we show this algorithm is both highly interpretable from a signal processing perspective and satisfies global optimality conditions based on results the literature } \cite{haeffele2019structured}{. Finally, we evaluate the effectiveness of the algorithm in four scenarios -- homogeneous vibrations, non-homogeneous vibrations, travelling plane waves, and segmented vibrations. We show that this framework outperforms most other methods used the literature for wave mode decomposition, particularly in the presence of high noise and segmented, spatially varying wave behavior. }

\section{Wave-Informed Matrix Factorization}
\label{others}

In this section, we develop a novel matrix factorization framework that allows us to decompose a matrix $\Y$, with each column representing sampling along one dimension (e.g., space), as the product of two matrices $\D \X^\top$, where one of the matrices will be {softly} constrained to {nearly} satisfy wave equation constraints.  Additionally, our model will learn the number of modes (columns in $\D$ and $\X$) directly from the data without specifying this \textit{a priori}. 

{In subsection} \ref{subsection:problem_formulation}, {we introduce our objective function. This subsection can be seen as an extension of the previously proposed wave-informed dictionary learning framework to a matrix factorization framework. In} \ref{section:model}, {we go over a few ideas from the optimization literature that are useful to solve the formulated problem. In} \ref{subsection:polar} {and} \ref{subsection:WIMF}, {we describe the main portions of the wave-physics informed matrix factorization algorithm.} {In} \ref{sec:sig_procc}, {we describe an interesting connection of this framework to signal processing. We emphasize that the novel parts of this work is in subsections} \ref{subsection:polar} {and} \ref{sec:sig_procc}. The work optimization discussions in \ref{section:model} {are specific cases of previous work and are shown for the benefit of the reader, particularly those in wave physics who may not be familiar with the optimization literature.}

\subsection{Problem Formulation}
\label{subsection:problem_formulation}
{The problem formulation described in this section is inspired from the wave-physics regularizer proposed in \cite{Tetali2019}. In this work, we additionally justify the wave-physics informed problem formulation in \cite{Tetali2019} as a naturally arising idea from the theory of solutions of partial differential equations.} A key concept in the solution of partial differential equations is the notion of separation.  For example, given a PDE in space $\ell$ and time $t$, we can often assume the solution is of the form 
\begin{align}
f(\ell,t) = d(\ell) x(t) \; . 
\end{align}
In the case of linear PDEs, we find that if $d_i(\ell) x_i(t)$ are solutions for all $i \in [1,N]$, then  
\begin{align}
f_i(\ell, t) = \sum_{i=1}^{N} {\alpha_i} d_i(\ell) x_i(t)
\end{align}
is also a solution {for any arbitrary choice of linear coefficients, $\{\alpha_{i}\}_{1=1}^N$}. For a discrete approximation of this solution, we note that the product of continuous functions $d_i(\ell) x_i(t)$ on $\ell \in [0,L]$ and $t \in [0,T]$ can be approximated as an outer product of two vectors, leading directly to a matrix factorization model, {where we further note that without loss of generality the scaling coefficients $(\alpha_i)$ can be absorbed into the factorized matrices}. 

Specifically, we consider matrices $(\D \in \Re^{N_d \times N},\X \in \Re^{N_t \times N})$ defined such that 
\begin{align}
\mathbf{d}_i &= \left[ d_i(0), d_i(\Delta \ell), \cdots, d_i(N_d \Delta \ell)  \right]^{\top} \nonumber \\
\mathbf{x}_i &= \left[ x_i(0), x_i(\Delta t), \cdots, x_i(N_t \Delta t)  \right]^{\top} \nonumber \\ 
\mathbf{D} \mathbf{X}^{\top} &= \sum_{i=1}^{N} d_i(\ell) x_i(t) = \sum_{i=1}^{N} \d_i \x_i^{\top}  \; ,
\end{align}
where $\d_i$ and $\x_i$ denote columns of $\D$ and $\X$, respectively, and $N_d = \left \lfloor{L/\Delta \ell}\right \rfloor$ and $N_t = \left \lfloor{T/\Delta t}\right \rfloor)$. 
Recall, we also wish to enforce physical consistency in our model, which we accomplish through theory-guided regularization (see \cite{karpatne2017theory}) in a cost function which 1) captures how our matrix factorization model matches the observed samples and 2) incorporates a regularization term which enforces physical consistency in the matrix factors. Specifically, we consider a model of the form
\begin{equation}
    \underset{\mathbf{D}, \mathbf{X}}{\text{ } \min \text{ }} { \tfrac{1}{2}\| \mathbf{Y} - \mathbf{D}\mathbf{X}^{\top} \|^2_F + \lambda \Theta (\mathbf{D}, \mathbf{X})}
    \label{eqn:obj_final}
\end{equation}
where $\| \mathbf{Y} - \mathbf{D}\mathbf{X}^{\top} \|^2_F$ is a data consistency loss, $\Theta (\mathbf{D}, \mathbf{X})$ is our physics-informed regularizer, and $\lambda > 0$ is a hyper-parameter to balance the trade-off between fitting the data and satisfying our model assumptions. 
When $\Theta (\mathbf{D}, \mathbf{X})$ is defined by the wave equation, we call this \emph{wave-informed matrix factorization}.

We derive a wave-informed $\Theta (\mathbf{D}, \mathbf{X})$ by considering the one-dimensional wave equation evaluated at $d_i(l)x_i(t)$
\begin{eqnarray}
    \frac{\partial^2 \left[ d_i(\ell)x_i(t) \right] }{\partial l^2} &=& \frac{1}{c^2} \frac{\partial^2 \left[ d_i(\ell)x_i(t) \right] }{\partial t^2} 
    \label{eqn:wave_equation}
\end{eqnarray}
The above equation constrains that each $d_i(\ell)x_i(t)$ component ($i \in [1,N]$) is a wave. This can be seen as decomposing the data into a superposition of simpler waves. Furthermore, the Fourier transform in time on both sides of \eqref{eqn:wave_equation} yields
\begin{eqnarray}
    \frac{\partial^2 \left[ d_i(\ell) \right] }{\partial l^2}X_i(\omega) &=& \frac{-\omega^2}{c^2} d_i(\ell) X_i(\omega) \\
    \frac{\partial^2 \left[ d_i(\ell) \right] }{\partial \ell^2} &=& \frac{-\omega^2}{c^2} d_i(\ell) 
    \label{eqn:space_eigen_value}
\end{eqnarray}
at points where $X_i(\omega) \neq 0$. This is specifically known as the Helmholtz equation. The Helmholtz equation is analytically solvable for constant $\omega/c$ and the solution takes the form $d_i(\ell) = A \sin(\omega \ell/c) + B \cos(\omega \ell/c)$, and the constants $A$ and $B$ can be determined by initial conditions. While the Helmholtz equation is often applied in the frequency domain, it is also applicable in time, where the relationship must be true for every frequency-velocity ($\omega,c$) mode in the data. Hence, if the data is sparse in these domains (as is often the case in modal analysis), then we can use the Helmholtz constraint to estimate and decompose all ($\omega,c$) modes in $\Y$. 

We discretize the Helmholtz equation such that
\begin{eqnarray}
    \mathbf{L} \d_i &=& -k_i^2 \d_i 
    \label{eqn:discrete_wave_eqn}
\end{eqnarray}
where $k_i = \omega_i / c_i$, also known as the angular wavenumber, and $\L$ is a discrete second derivative operator. The specific form of $\L$ depends on the boundary conditions (e.g., Dirichlet, Dirichlet-Neumann, etc), which can be readily found in the numerical methods literature (see for example \cite{strang2007computational, golub2013matrix}).
%
In this paper, we use the Laplacian matrix ($\L$) defined by
\begin{eqnarray}
    \mathbf{L} &=& \frac{1}{(\Delta l)^2} \begin{bmatrix}
    -2 & 1 & 0 & 0 & 0 &\cdots & 0 \\
    1 & -2 & 1 & 0 & 0 & \cdots & 0 \\
    0 & 1 & -2 & 1 & 0 & \cdots & 0 \\
    \vdots & \vdots & \vdots & \vdots & \vdots & \ddots & \vdots \\
    0 & 0 & 0 & 0 & 0 & \cdots & -2 \\
    \end{bmatrix}.
    \label{eqn:Lap_mat}
\end{eqnarray}
Note that to reduce notational complexity, we assume $\Delta l = 1$  without any loss of generality, as any change in $\Delta l$ can be absorbed into the value of $k_i$ after plugging \eqref{eqn:Lap_mat} into \eqref{eqn:discrete_wave_eqn}.  

We also specify the dependence on $i$ owing to the fact that $\omega/c$ and $d_i(l)$ determine each other. 
For the factorization to satisfy the wave equation, we desire that \eqref{eqn:discrete_wave_eqn} is satisfied for some value of $k_i$, which we promote with the regularizer 
\begin{align}
\min_{k_i} \| \mathbf{L} \d_i + k_i^2 \d_i \|^2_F \; .
\end{align}
Note that {for the above regularizer it is easily shown that the optimal value for $k_i^2$ must lie in the range $[0, \frac{4}{\Delta l}]$, where $\frac{4}{\Delta l}$ is the smallest eigenvalue of $-\L$ \cite{chung2000discrete}. 
Since we specifically chose $\Delta l = 1$ (without loss of generality), we have that 
the optimal value of $k_i^2$ must lie in the range $[0,4]$.}

In addition to requiring that the columns of $\D$ satisfy the wave equation, we also would like to constrain the number of modes to be minimal.  We accomplish this by adding the squared Frobenius norms to both $\D$ and $\X$, which is known to induce low-rank solutions in the product $\D \X^\top$ due to connections with the variational form of the nuclear norm \cite{haeffele2019structured,haeffele2014structured,srebro2005rank}. Hence, our complete regularization is defined by
\begin{align}
\label{eq:theta}
\Theta(\D,\X,\k) &= \tfrac{1}{2} \sum_{i=1}^N \bar \theta(\d_i,\x_i,k_i)  \\
\label{eq:theta2}
  \bar \theta(\d_i,\x_i, k_i) &= \left( \|\x_i\|_F^2  + \|\d_i\|_F^2 \right)  +  \!  \gamma \min_{k_i} \|\L \d_i \!  + \! k_i^2 \d_i \|_F^2  \nonumber \\
   &= \|\x_i\|_F^2  + \min_{k_i} \d_i^\top \left( \I + \gamma (\L + \I k_i^2)^2  \right)  \d_i \nonumber \\
   &= \|\x_i\|_F^2  + \min_{k_i} \d_i^\top \A(k_i) \d_i  \\
\label{eq:Ak}
   \text{where} \ \A(k_i) &= \I + \gamma (\L + \I k_i^2)^2 .
\end{align}
%
With this regularization function, the complete cost function is defined by
\begin{align}
\label{eq:main_obj}
\min_{\substack{\D, \X, \k, N}} \! \tfrac{1}{2} \|\Y-\D\X^\top\|_F^2 \! + \! \tfrac{\lambda}{2} \! \|\X\|_F^2 \! + \! \tfrac{\lambda}{2} \! \sum_{i=1}^N \d_i^\top \A(k_i) \d_i 
\end{align}
where note that we are also optimizing over the number of modes $(N)$ and the wavenumber $\mathbf{k} = \left[ k_1, \cdots, k_N \right]^{\top}$ for each mode in the data.

\subsection{Model Optimization}
\label{section:model}
Our model in \eqref{eq:main_obj} is inherently non-convex in $(\D,\X,\k)$ due to the matrix factorization model and further complicated by the fact that we are additionally searching over the number of columns/modes $N$. However, despite the challenge of non-convex optimization, we can solve \eqref{eq:main_obj} by leveraging prior results from optimization theory for structured matrix factorization \cite{haeffele2014structured,bach2013convex}. {Note that, for the remainder of this subsection we reiterate the concepts and framework described in }\cite{haeffele2014structured,haeffele2019structured} {leading to adapting the algorithm mentioned therein to our specific model and finally obtaining an optimal solution of the objective function} \eqref{eq:main_obj}. For a matrix factorization problem of the form
\begin{equation}
\label{eq:gen_obj}
\min_{N}{\min_{\D, \X}} \; L(\D \X^\top) + \lambda \sum_{i=1}^N \bar \theta(\d_i,\x_i)
\end{equation}
where $L(\widehat \Y)$ is any function which is convex and once differentiable in $\widehat \Y$ and $\widebar \theta(\d,\x)$ is any function that satisfies the following three conditions
\begin{enumerate}
\item $\widebar \theta(\alpha \d, \alpha \x) = \alpha^2 \widebar \theta(\d, \x), \ \forall (\d,\x)$, $\forall \alpha \geq 0$.
\item $\widebar \theta(\d, \x) \geq 0, \ \forall (\d,\x)$.
\item For all sequences $(\d^{(n)},\x^{(n)})$ such that $\|\d^{(n)}(\x^{(n)})^\top\| \rightarrow \infty$ then $\widebar \theta(\d^{(n)},\x^{(n)}) \rightarrow \infty$ ,
\end{enumerate}
an algorithm exists that can obtain the global minimum solution. In supplementary material~\ref{ap:global_optimal}, we show that the wave-informed cost function satisfies all three conditions, where the optimization over the $k_i$ parameters is done inside $\theta$ as in \eqref{eq:theta2}. 
%


Moreover, in \cite{haeffele2019structured} it is shown that a given point $(\widetilde \D, \widetilde \X)$ is a globally optimal solution of \eqref{eq:gen_obj} iff the following two conditions are satisfied:
\begin{enumerate}
    \item$
 \langle -\nabla L( \widetilde \D \widetilde \X^\top), \widetilde \D \widetilde \X^\top \rangle = \lambda \sum_{i=1}^N \bar \theta(\widetilde \d_i, \widetilde \x_i) 
$
\item $
 \Omega_{\bar \theta}^\circ (-\tfrac{1}{\lambda} \nabla L(\widetilde \D \widetilde \X^\top)) \leq 1
$
\end{enumerate}
where in our scenario
\begin{align}
    -\tfrac{1}{\lambda} \nabla(L(\D \X^{\top})) &= \tfrac{1}{\lambda} \left( \Y - \D \X^\top \right) \nonumber
\end{align}
denotes the gradient w.r.t. the matrix product and $\Omega_{\bar \theta}^\circ(\cdot)$ is referred to as the polar problem, which is defined as 
\begin{equation}
\label{eq:polar_def}
\Omega_{\bar \theta}^\circ (\Z) \equiv \sup_{\mathbf{d},\mathbf{x}} \mathbf{d}^\top \Z \mathbf{x} \st \bar \theta(\mathbf{d},\mathbf{x}) \leq 1.
\end{equation}
It is further shown in \cite{haeffele2019structured} (Proposition 3)  that the first condition above will always be satisfied for any first-order stationary point $(\widetilde \D, \widetilde \X)$, and that if a given point is not globally optimal then the objective function \eqref{eq:gen_obj} can always be decreased by augmenting the current factorization by a solution to the polar problem as a new column:
%
\begin{align}
\label{eq:polarproblem}
(\D, \X)  &\leftarrow \left( \left[ \widetilde \D, \ \tau \mathbf{d}^* \right], \left[ \widetilde \D, \ \tau \mathbf{x}^* \right] \right)  \!\\
\!  \mathbf{d}^*, \mathbf{x}^* &\in \argmax_{\mathbf{d},\mathbf{x}} \tfrac{1}{\lambda} \mathbf{d}^\top \! \left(  \Y - \D \X^\top \right) \x \st \bar \theta(\mathbf{d},\mathbf{x}) \leq 1   \nonumber 
\end{align}
for an appropriate choice of step size $\tau > 0$.  

Despite the fact that \cite{haeffele2019structured} gives a meta-algorithm that is guaranteed to solve \eqref{eq:gen_obj} to global optimality, one needs to solve the polar problem in \eqref{eq:polar_def} to both verify global optimality and escape non-optimal stationary points, which is itself another optimization problem. This presents a potentially significant computational challenge in practice, as for many choices of regularization function, solving the polar problem is NP-hard~\cite{bach2013convex}. For example, if the regularization on the matrix columns is simply the sum of the $\ell_2$ norms of the columns, then the polar problem is easily solved as a maximum eigenvalue problem, but if we either replace one of the $\ell_2$ norms by the $\ell_\infty$ norm, the problem becomes NP-hard \cite{bach2013convex}.  Fortunately, however, in the next section we show that the polar problem that results from our wave-informed factorization model is tractable and leads to polynomial time global optimality.


\subsection{Solving the Polar Problem}
\label{subsection:polar}

Specifically, for our regularization function \eqref{eq:theta}, we show the following result which allows the polar problem to be solved efficiently, enabling efficient and guaranteed optimization  \eqref{eq:main_obj}.
\begin{theorem}
\label{thm:polar}
For \eqref{eq:main_obj}, the polar problem in \eqref{eq:polar_def} is 
\begin{align}
\Omega_\theta^\circ (\Z) = & \max_{\d, \x, k} \, \d^\top \Z \x  \nonumber \\
&\mathrm{s.t.} ~ \d^\top \A(k) \d \leq 1, \|\x\|^2_F \leq 1, \ 0 \leq k \leq 2. \nonumber
\end{align}
where $\A(k)$ is as defined in \eqref{eq:Ak}.
Further, if we define $k^*$ as%
\begin{equation}
k^* = \argmax_{k \in [0,2]} \| \A(k)^{-1/2} \Z \|_2 \; . 
\label{eq:k_max}
\end{equation}
Then the optimal values of $\d,\x, k$ are given as $\d^* = \A(k^*)^{-1/2} \widebar \d$, $\x^* = \widebar \x$, and $k^*$.  Where $\widebar \d$ and $\widebar \x$ are the left and right singular vectors, respectively, associated with the largest singular value of $\A(k^*)^{-1/2} \Z$. 
\end{theorem}
%
%
%

The proof of Theorem 1 is detailed in supplementary material~\ref{ap:polar_problem}. The result implies that we can solve the polar by performing a one-dimensional line search over $k$. Due to the fact that the largest singular value of a matrix is a Lipschitz continuous function, this line search can be solved efficiently by a variety of global optimization algorithms. In particular, the above line search over $k$ is Lipschitz continuous (with respect to $k^2$) with a Lipschitz constant, $L_{k}$, which is bounded by
\begin{align}
L_{k} \! &\leq \! \begin{Bmatrix} \frac{2}{3 \sqrt{3}} \sqrt{\gamma} \| \Y - \D \X^\top \|_2 & \gamma \geq \frac{1}{32} \\ 
4 \gamma (1 + 16 \gamma)^{-\tfrac{3}{2}} \| \Y - \D \X^\top \|_2 & \gamma < \frac{1}{32} \end{Bmatrix}  
\nonumber \\ 
&\qquad\quad \leq \tfrac{2}{3 \sqrt{3}} \sqrt{\gamma} \|\Y - \D \X^\top\|_2 \!
\end{align}
This is formally stated and proven proven in supplementary material~\ref{ap:finite_computations} as Theorem~\ref{thm:Lip_line_search}, where we also give a formal result, in Corollary~\ref{cor:line_search}, which guarantees that the line search over $k$ (and hence the overall polar problem) can be solved in polynomial time. Building from this key result, next we discuss an algorithm to provably solve our wave-informed factorization model in polynomial time. 

%

\subsection{Wave-Informed Factorization Algorithm}
\label{subsection:WIMF}
Algorithm~\ref{alg:meta} defines our wave-informed matrix factorization algorithm. The algorithm has three components that are iterated: (1) perform gradient descent with a fixed number of modes/columns, $N$, to reach a first order stationary point for $\D$, $\X$, and $k$,  (2) solve the polar problem via Theorem \ref{thm:polar} to obtain $\d^*$, $\x^*$, and $k^*$, and (3) if the stopping condition is not met, update $\D$, $\X$, and $k$ by appending the polar solution $(\d^*,\x^*, k^*)$ to the solution (scaled by an optimal step size $\tau$); otherwise, terminate the algorithm.

\begin{algorithm}
\caption{\bf{Wave-Informed Matrix Factorization}}
\label{alg:meta}
\begin{algorithmic}[1]
\STATE Input $\mathbf{D}_{init}$, $\mathbf{X}_{init}$, $\mathbf{k}_{init}$, stopping tolerance: $\epsilon$ 
\STATE Initialize $ \left( \mathbf{D}, \mathbf{X}, \mathbf{k} \right) \leftarrow \left( \mathbf{D}_{init}, \mathbf{X}_{init}, \mathbf{k}_{init} \right)$ 
\STATE Initialize $ \left( \tilde{\mathbf{D}}, \tilde{\mathbf{X}}, \tilde{\mathbf{k}} \right) \leftarrow \left( \mathbf{D}_{init}, \mathbf{X}_{init}, \mathbf{k}_{init} \right)$
\WHILE {$\left| \ \Omega^\circ_\theta(\tfrac{1}{\lambda}(\Y-\widetilde \D \widetilde \X^\top)) - 1 \ \right| \geq \epsilon$}
\STATE Perform gradient descent on \eqref{eq:main_obj} to obtain $(\widetilde{\mathbf{D}}, \widetilde{\mathbf{X}}, \widetilde{\mathbf{k}})$\label{step:grad_desc}
\STATE Calculate $\Omega^\circ_\theta(\tfrac{1}{\lambda}(\Y-\widetilde \D \widetilde \X^\top))$ and obtain $\mathbf{d}^*, \mathbf{x}^*, k^*$
\IF {$ \Omega^\circ_\theta(\tfrac{1}{\lambda}(\Y-\widetilde \D \widetilde \X^\top)) - 1  > \epsilon$} 
\STATE $ \left( \mathbf{D}, \mathbf{X}, \mathbf{k} \right) \! \leftarrow \! \left( \left[ \widetilde{\mathbf{D}}, \ \tau \mathbf{d}^*  \right] \!, \! \left[ \widetilde{\mathbf{X}}, \ \tau \mathbf{x}^*  \right] \! , \! \left[ \widetilde{\mathbf{k}}^\top, \ k^*  \right]^\top
\right)\!$
\ENDIF
\ENDWHILE
\end{algorithmic}\
\label{algoblock:meta-algo}
\end{algorithm}
\

\subsubsection{Gradient Descent Update} 
We begin our algorithm by performing block coordinate descent 
with respect to $\D$ and $\X$ along with a block minimization step with respect to $k$ on the objective in \eqref{eqn:obj_final}, which we note is guaranteed to reach a first-order stationary point  \cite{xu2013block}. The number of columns of $\D$ and $\X$ (denoted by $N$ in \eqref{eq:main_obj}) is fixed and does not change during this step. Specifically, we use the following update equations (for a step size $\alpha$)
\begin{align}
    \d_i^+ &= \d_i  - \alpha \left(  \left(\D \X^{\top} - \Y \right) \x_i  +  \lambda \d_i \right. \nonumber \\ &\qquad\quad  \left. + 2 \gamma \lambda \left( \L + k_i^2 \I  \right)^2 \d_i \right) \label{eq:grad_D} \\
    \x_i^+ &= \x_i - \alpha \left( \left( \D^+ \X^{\top} - \Y \right)^{\top} \d_i^+ + \lambda \x_i \right) \label{eq:grad_X} \\
    k^+_i &= \sqrt{-\frac{ (\d^+_i)^\top \L \d^+_i } {\|\d^+_i\|_2^2}} \qquad  \textrm{for each $i=1,\ldots,N$} \label{eq:grad_k}
\end{align}
where $\D^+$, $\X^+$, and $k_i^+$ are the updated values and the $i$ subscript denotes a column of the respective matrix. 



\subsubsection{Solve Polar Problem}  
After finding a first order stationary point $(\widetilde{\D},\widetilde{\X},\widetilde{\k}$, we next solve the polar problem according to Theorem \ref{thm:polar} with $\Z = \tfrac{1}{\lambda}(\Y - \widetilde{\D} \widetilde{\X}^\top)$. In particular, we first solve the one-dimensional line search for the optimal value of $k$ in \eqref{eq:k_max}, where this line search can be solved using many approaches.  For example, the LIPO algorithm given in \cite{malherbe2017global} can provably solve the line search in polynomial time. 
Once the optimal $k^*$ is obtained, the values of $\d^*$ and $\x^*$ are then computed directly from the eigen-decomposition of $\A(k^*)$ as described in Theorem \ref{thm:polar}. 

\subsubsection{Stopping Condition and Growing the Factorization}  {After solving the polar problem as described above, the value of the polar gives a test for checking global optimality, and if this test is not met then we can escape the first order stationary point by appending the polar solution to our factorization.  Specifically, once we evaluate the polar $\Omega^\circ_\theta(\tfrac{1}{\lambda}(\Y-\widetilde \D \widetilde \X^\top))$, then in \cite{haeffele2019structured} (Prop. 4) it is shown that the distance to the global optimum (in objective value) is directly proportional to the value of the polar value minus 1 (and that the value of the polar is always $\geq 1$), so choosing to stop the algorithm when the value of polar takes a value $\leq 1 + \epsilon$ also guarantees optimality to within $\mathcal{O}(\epsilon)$.  If the value of the polar does not meet our stopping condition, then we continue the optimization by appending the polar solution to our factorization after an optimal scaling.}  Specifically, let
\begin{equation}
    \label{eqn: updates}
    \left(\D_\tau, \X_\tau, \k_{\tau} \right) \! = \! \left( \left[ \widetilde{\mathbf{D}} , \ \tau \mathbf{d}^*  \right] \! , \! \left[ \widetilde{\mathbf{X}}, \ \tau \mathbf{x}^*  \right] \! , \! \left[ \widetilde{\mathbf{k}}^\top, \ k^*  \right]^\top \!
\right)
\end{equation}
be our optimization variables with the new polar solution appended to the factorization scaled by $\tau$, then we find the optimal value of $\tau$ to minimize the original objective function:
%
\begin{align}
 \label{eqn:for_tau}
    &\min_{\substack{\tau}} \! \tfrac{1}{2} \|\Y-\D_{\tau}\X_{\tau}^\top\|_F^2 \! + \! \tfrac{\lambda}{2} \! \|\X_{\tau}\|_F^2 \! + \! \tfrac{\lambda}{2} \! \sum_{i=1}^N \D_{\tau}^\top \A(k_{i,\tau}) \D_{\tau} 
\end{align}
It is shown in supplementary material~\ref{appendix:optimal_step_size} that the optimal $\tau$ is given by
\begin{eqnarray}
    \tau = \frac{\sqrt{(\d^*)^\top \left( \Y - \widetilde{\D} \widetilde{\X}^{\top} \right) \x^* - \lambda }}{\|\d^*\|_2 \|\x^*\|_2 }
    \label{eq:for_tau_}
\end{eqnarray}
After updating the variables as in \eqref{eqn:for_tau} with the optimal choice of $\tau$, we then return to gradient descent updates until reaching a first order stationary point. 

{
Taken together, our proposed algorithm results in a polynomial-time algorithm for solving our main objective} \eqref{eq:main_obj}.  {We do not give a formal result as the algorithm given in Algorithm} \ref{alg:meta} 
{is simply one choice of optimization algorithm which is possible due to the polynomial time solution for the polar problem that we give in Theorem} \ref{thm:polar}
{and the overall convergence of the algorithm is largely known from prior work, provided the polar problem can be solved.  For example, the authors of} \cite{xu2013block} 
{prove that the block coordinate gradient descent portion of the algorithm will converge to a stationary point at a rate of $\mathcal{O}(1/t)$ (with $t$ being the number of iterations), which can be improved to $\mathcal{O}(1/t^2)$ with Nesterov acceleration.  Likewise, the authors of}  \cite{bach2013convex} 
{ show that a Frank-Wolfe/conjugate-gradient step (i.e., when we escape poor stationary points by solving the polar problem) will reduce the objective with a linear convergence rate.  From this it is straight-forward to combine these arguments to give an overall guarantee of convergence, but we do not do so here to limit the scope and length of the manuscript.}
%



\subsection{Signal Processing Interpretation of the Regularizer}
\label{sec:sig_procc}
{We now discuss an interesting interpretation of the above algorithm from a signal processing perspective. Specifically, if we let $\L = \Gamma \Lambda \Gamma^\top$ denote a singular value decomposition of $\L$, then note that} when identifying the optimal $k$ value in the polar program, we solve for
\begin{equation}
\label{eq:k_opt}
\argmax_{k \in [0,2]} \| \Gamma (\I + \gamma ( k^2 \I +  \Lambda)^2)^{-1/2} \Gamma^\top \Z \|_2 \; .
\end{equation}
This optimization has an intuitive interpretation from signal processing. Given that $\Gamma$ contains the eigenvectors of a Toeplitz matrix, those eigenvectors have spectral qualities similar to the discrete Fourier transform (the eigenvectors of the related circulant matrix would be the discrete Fourier transform \cite{DBLP:books/lib/OppenheimS75, DBLP:journals/siamrev/Strang89}). As a result, $\Gamma^\top$ transforms the data $\Z$ into a spectral-like domain and $\Gamma$ returns the data back to the original domain. Since the other terms are all diagonal matrices, they represent element-wise multiplication across the data in the spectral domain. This is equivalent to a filtering operation, with filter coefficients given by the diagonal entries of $(\I + \gamma ( \bar k \I +  \Lambda)^2)^{-1/2}$. Figure~\ref{fig:the_filter} shows two such filter responses at $k^2=2.5, \gamma = 1000$ and $k^2 = 1.5, \gamma = 10000$. Observe that, an increased value of $\gamma$ has reduced the bandwidth of the filter.

\begin{figure}[th]
    \centering
    \includegraphics[scale=0.35]{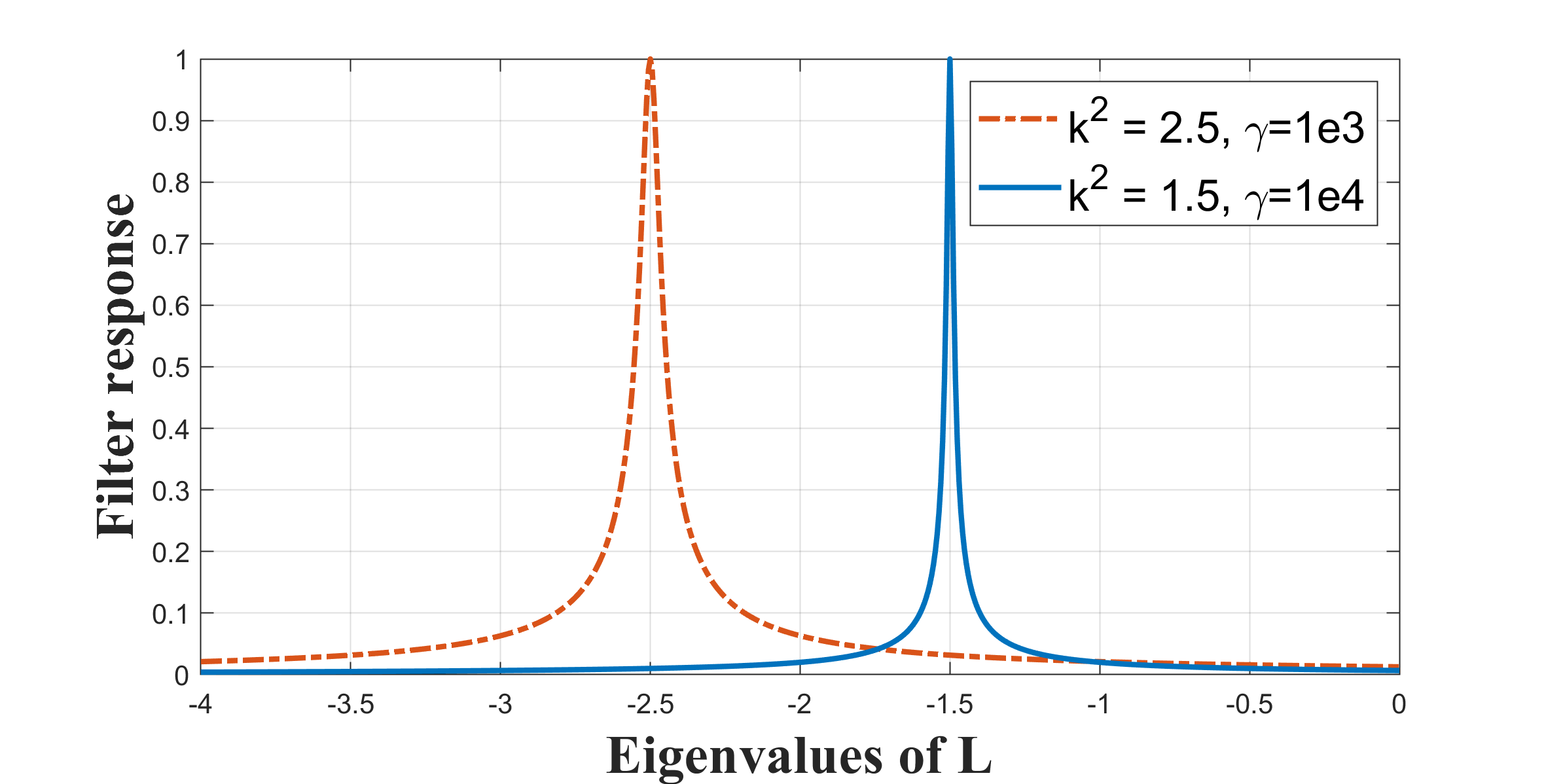}
    \caption{Butterworth Filter on the spectral-like domain}
    \label{fig:the_filter}
\end{figure}

Furthermore, recall that the transfer function of a 1st-order Butterworth filter is given by:
\begin{equation}
T(\omega) = \frac{1}{\sqrt{1+\gamma (\omega_0 + \omega)^2}}
\end{equation}
where $\omega_0$ is the center frequency of the passband of the filter and $1/\sqrt{\gamma}$ corresponds to the filter's $-3$dB cut-off frequency.  Comparing this to the filter coefficients from \eqref{eq:k_opt}, we note that the filter coefficients are identical to those of the 1st-order band-pass Butterworth filter, where $\Lambda$ corresponds to the angular frequencies. 

As a result, we can consider this optimization as determining the optimal filter center frequency $\omega_0=k^2$ with fixed bandwidth $(1/\sqrt{\gamma})$  that retains the maximum amount of signal power from $\Z$.  Likewise, the choice of the $\gamma$ hyperparameter sets the bandwidth of the filter. As $\gamma \rightarrow \infty$, the filter bandwidth approaches 0 and thereby restricts us to a single-frequency (i.e., Fourier) solution. 
Furthermore, we can provide a recommended lower bound for $\gamma$ according to $\gamma > 1/k_{bw}^2$, where $k_{bw}$ is the bandwidth of the signal within this spectral-like domain. {We would like to emphasize that this work introduces a unique theme within the works of scientific machine learning. Through means of wave-informed learning, we obtain a new method to mathematically derive the well-known \emph{digital Butterworth filter}}.

\section{Experimental Evaluation}
\label{sec:data}


\begin{figure*}[bt]
    \centering
    \includegraphics[scale=0.57]{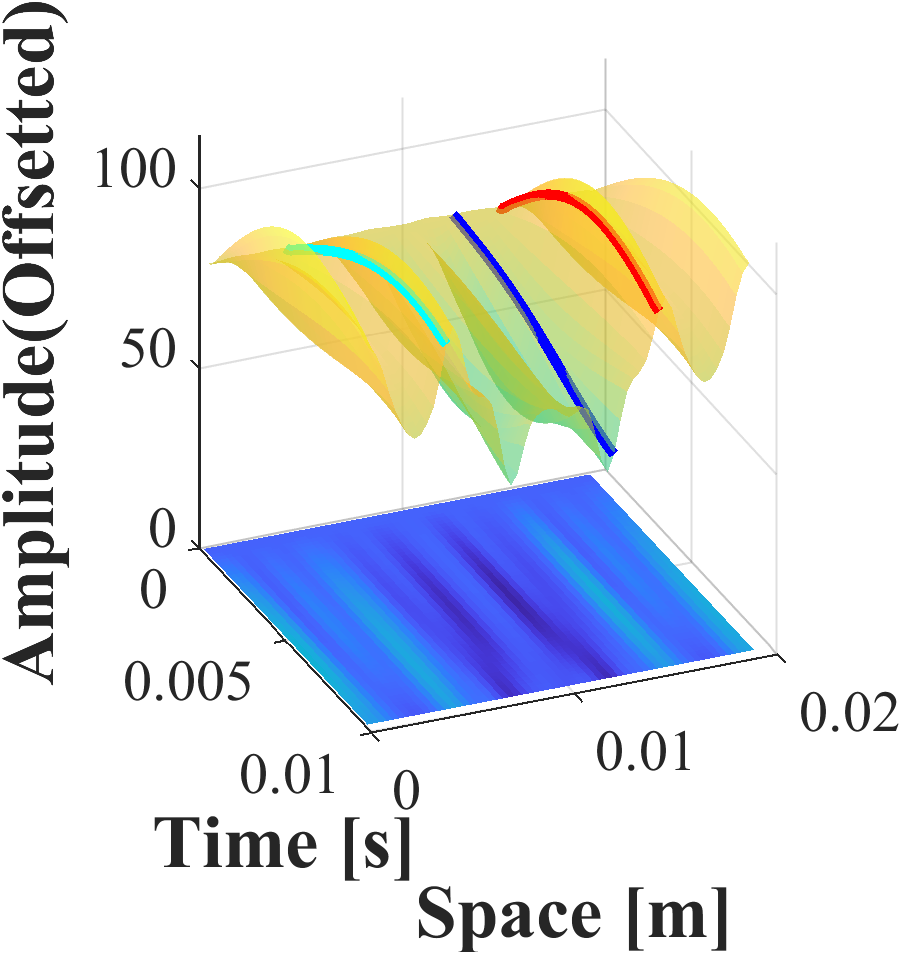}
    \includegraphics[scale=0.57]{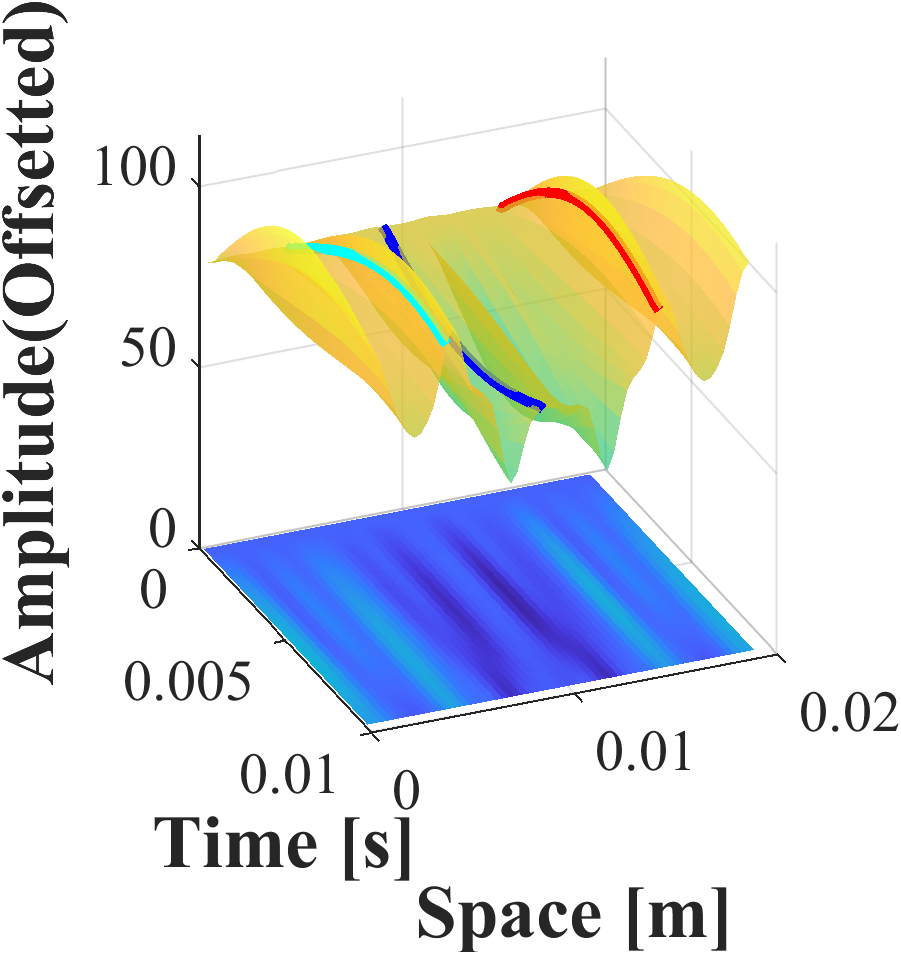}
    \includegraphics[scale=0.57]{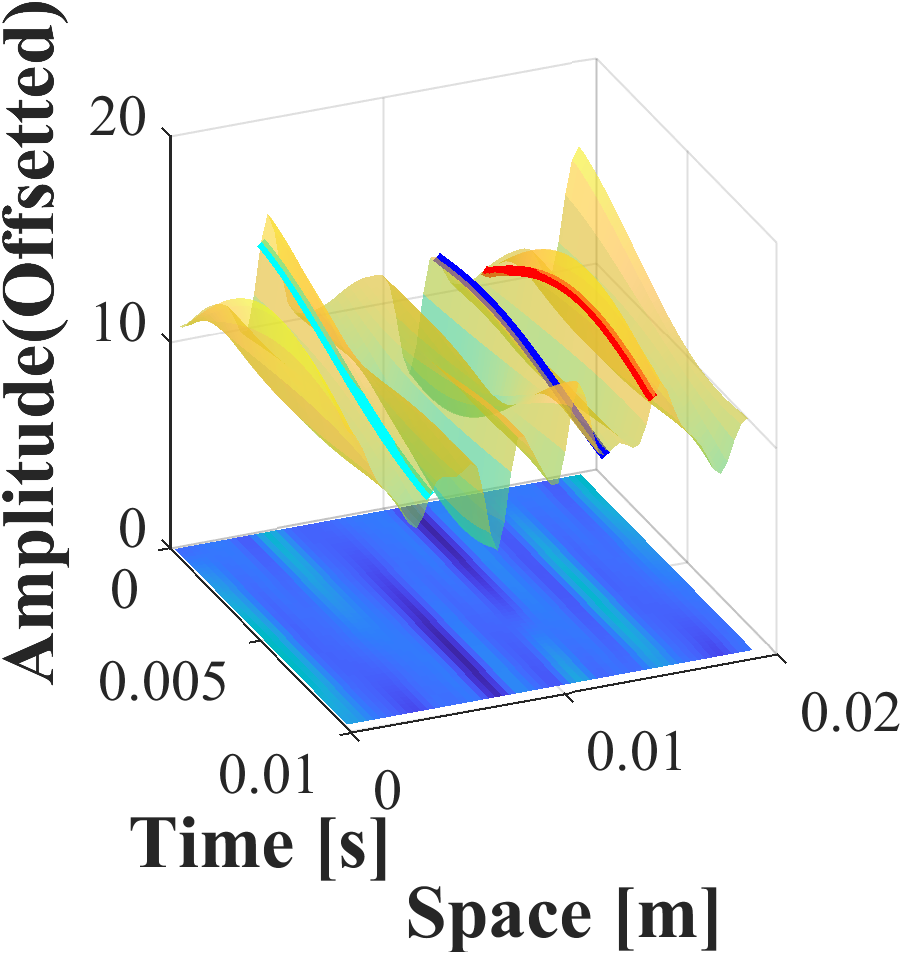}
    \includegraphics[scale=0.57]{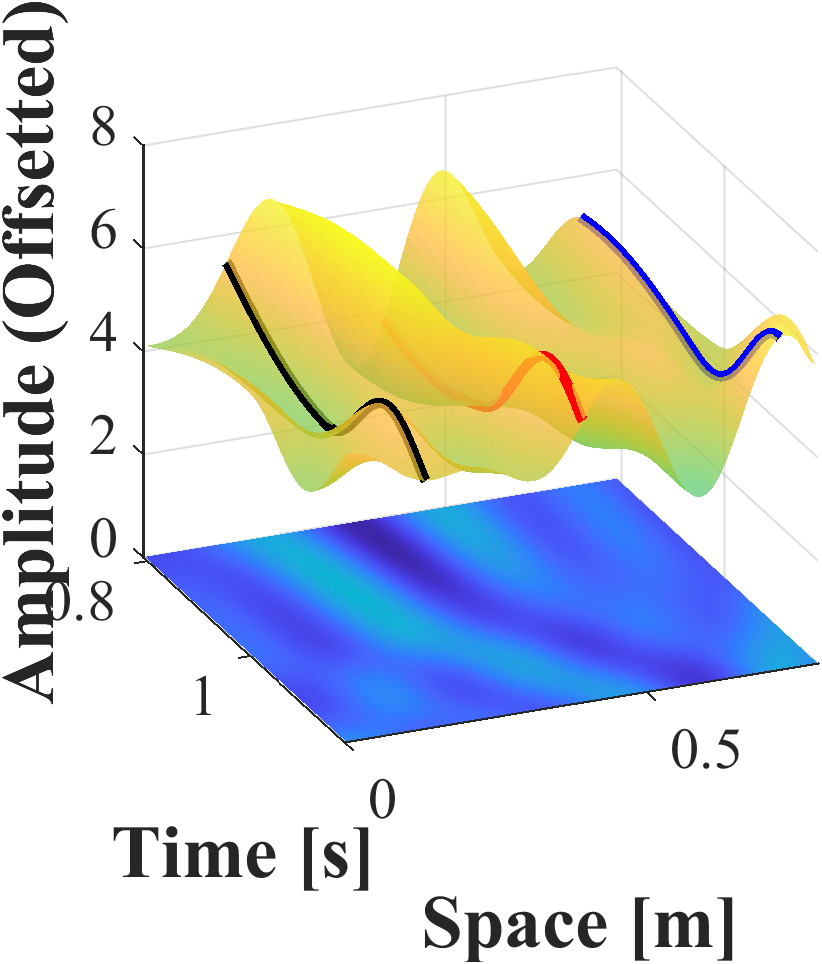}\\
    (a) \hspace{1.5in}
    (b) \hspace{1.5in}
    (c) \hspace{1.5in}
    (d)
    
    \caption{(a) Homogeneously vibrating medium marked with different colors at three different time instances for a short spatial interval at the start. (b) Non-Homogeneously (exponentially decaying over space) vibration data marked with different colors at three different time instances for a short spatial interval at the start. (c) Traveling wave data marked with different colors at three different time instances for a short spatial interval at the start. (d) Data of two mediums joined end-to-end and vibrating together marked with different colors at different time instances for a short spatio-temporal interval around the joint.}
    \label{fig:the_data}
\end{figure*}

We demonstrate the results from wave-informed matrix factorization on four simulated vibration datasets, each characterized by increasing complexity: (1) a homogeneous vibration with fixed boundary conditions, (2) an non-homogeneous vibration with fixed boundary conditions and spatial amplitude decay, (3) a plane wave traveling with temporal decay, and (4) a vibration in a segmented medium, where each segment has an independent wavenumbers and amplitude. In all four datasets, one dimension of the matrix represents space and the other dimension represents time, as is common in modal analysis \cite{osti_1487358, YANG2017567}. Multiple modes share the same spatial locations. Fig.~\ref{fig:the_data} illustrates these four datasets.

For each method and dataset, we study five different signal-to-noise ratio (SNR) levels. Each signal is defined by 
\begin{align}
    \widehat{y}(\ell,t) = y(\ell,t) + \eta(\ell,t) \; ,
\end{align}
where $y(\ell,t)$ represents our data and $\eta(\ell,t)$ represents our noise. 
The noise is assumed to be additive white Gaussian noise across both space and time. The SNR is defined by 
\begin{align}
    \textrm{SNR} = \frac{\sum_{\ell,t} |y(\ell,t)|^2 }{\sum_{\ell,t} |\eta(\ell,t)|^2} \; .
\end{align}
The SNR is assumed to be unknown to each of the algorithms. 

The following subsections describe how each of our datasets are computed. Table~\ref{table:params_selected} lists the parameters for each of these scenarios. Note that each case can have potentially an infinite number of modes $N$. In practice, these modes usually attenuate as the associated frequency / wavenumber increases. To simplify this, we simulate a finite number of modes for each dataset. For the segmented data, $k_n$ represents the first $8$ theoretical modes (four in each segment) for vibration across two connected segments. 

To solve the polar problem and estimate $k^*$ in Algorithm~\ref{algoblock:meta-algo}, many approaches can be taken. In this paper, we solve for $k^*$ by first coarsely searching for a maximum among coarsely uniformly placed points. For a predetermined region around the maximum, we perform a fine search by taking the maximum over finely uniformly placed points in that region


\begin{table*}[t]
\caption{\label{table:params_selected} Table of parameters used to generate data, with additive noise $\eta_n \sim \mathcal{N}(0,1)$ and $u_n \sim \mathcal{U}[0,1]$. The random variables are independently sampled for each parameter and $1 \leq n \leq N$.}

\centering
\begin{tabular}{lllllllllll}
\hline
Data    & $(\Delta \ell, L)$ & $(\Delta t,T)$ & $\alpha_n$ & $\beta_n$     & $k_n$      & $\omega_n$ &  N & $\delta$ \\
\hline
Homogeneous Vibration & $(0.0901,1)$   & $(0.0005,2)$      & $\frac{n}{2}$    & $0$ & $ n \pi$ & $106k_n$  & 6 & 100 \\
Non-homogeneous Vibration & $(0.901,10)$   & $(0.0005,2)$           & 0   & $\frac{4+n}{10} + u_n$ & $ n \pi$ & $106k_n$  &  6 & 1 \\
Travelling Plane Wave & $(0.901,10)$   & $(0.0005,2)$     &  $\frac{n}{2}$   & 0 & $ n \pi$& $106k_n$  &  6  &  1 \\
Segmented Vibration & $(0.01,2)$ & $(0.02,10)$ & - & - & $2\pi \pm \arctan \left( \sqrt{ \frac{1}{3} \left( 13 \pm 4\sqrt{10} \right)  } \right)$  & $3k_n$  & $8$ &  1 \\%
\hline
\end{tabular}
\end{table*}

\subsection{Homogeneous Vibration Data with Fixed Boundaries}
Our first dataset emulates vibrations in a medium with fixed (Dirichlet) boundary conditions such that $y(0,t) = 0$ and $y(L,t)=0$, where $L$ is the length of the medium. The sinusoids reflect from each end and create a standing wave in space. Such data is defined by
\begin{eqnarray}
y(\ell,t)\!=\!\sum_{n=1}^{N} e^{-\alpha_n t} \sin (n k \ell) \left( a_n\sin(n \omega t )\!+\!b_n\cos(n \omega t) \right) \, .
\label{eqn:homogenous_vibration}
\end{eqnarray}
This type of data is commonly found in vibrations and modal analysis problems \cite{Ma2019PhysicalIO, lai2020full}. For example, the modes and their parameters can be used to determine  material or structural characteristics \cite{articleKammer, articleCunha, Marshall1985ModalAO}. Apart from the added noise, we assume a damping over time to introduce a non-ideality. For this illustration we choose $L=1$.

\subsection{Inhomogeneous Vibration Data With Spatial Decay}

The previous dataset assumed spatial homogeneity. That is, the behavior of the medium (i.e., the wavenumber and amplitude) did not vary across space $\ell$. This assumption may not always be true. As a result, we would like to extract modes with inhomogenieties. We consider the same fixed boundary data but with a spatial decay such that
\begin{eqnarray}
y(\ell,t) = \sum_{n=1}^{N}  e^{- \beta_n \ell}  \sin ( k_n \ell) \left( a_n \sin(n \omega t) + b_n \cos(n \omega t) \right)  \,.
\label{eqn:non-homogenous_vibrations}
\end{eqnarray}
In this expression, we have introduced a new parameter $\beta_n$ that causes the amplitude to decay over space, similar to an evanescent wave in electro-magnetics and is usually meant to persist over longer regions of space. This example is meant to explore if each algorithm addresses amplitude non-homogeneities. In contrast to the previous case, where we considered measuring vibrations along a string of length 1, we consider a scenario of measuring waves for a longer length. In this illustration, we fix a length of 10 for measuring waves decaying exponentially over space. As an illustration for the advantage of length, the spectrum (computed by multiplying data with the transpose of the eigenvectors $\Gamma$) of three exponentially decaying sinusoids (for length 1 and length 10) corresponding to $k_1 = \pi$, $k_2 = 2\pi$, and $k_3 = 3 \pi$ are shown in Figure~\ref{fig:decaying_spectra}(a) and Figure~\ref{fig:decaying_spectra}(b), respectively. In each plot, the peaks correspond to the same three values of $k$, but due to the 10 times increases in length, Figure~\ref{fig:decaying_spectra}(b) exhibits a 10 times improvement in resolution in the wavenumber domain. In the presence of spectral spread (e.g., caused by a spatial decay), this can improve our ability to distinguish each mode.  The parameters generating the data are given in Table~\ref{table:params_selected}. 


\begin{figure*}
    \centering
    \includegraphics{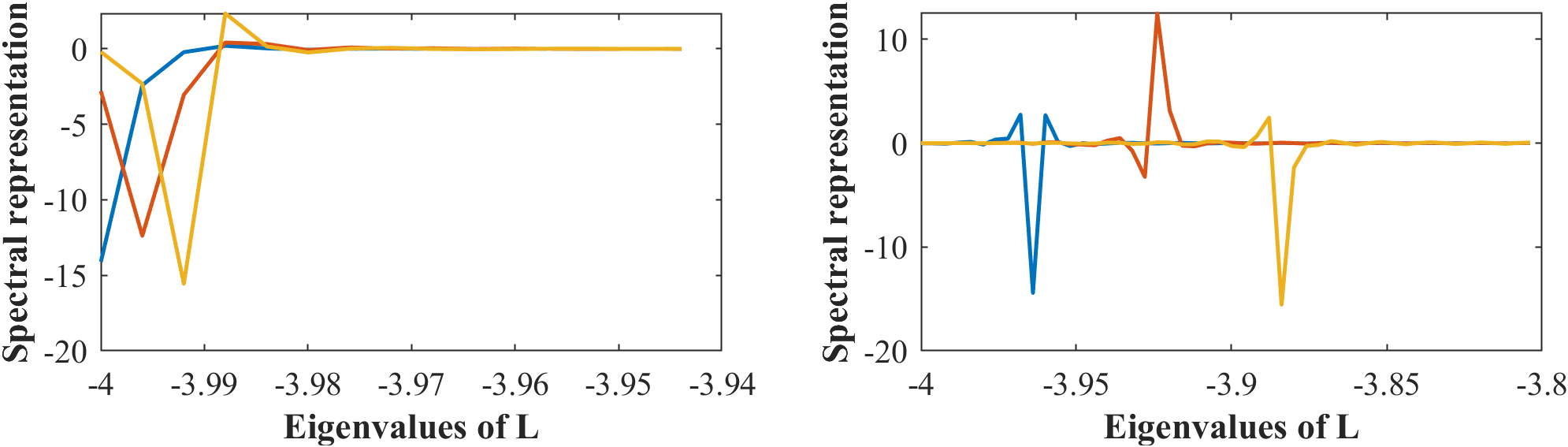}
    ~\\~(a) \hspace{24em} (b)
    \caption{ Spectra (with respect to Eigen-basis of $\L$) of decaying exponentials $e^{-\beta_1 \ell} \sin (\pi \ell)$, $e^{-\beta_2 \ell} \sin (2\pi \ell)$ and $e^{-\beta_3} \sin(3\pi \ell)$ (indicated using blue, red and yellow colors respectively) for (a) $ \ell \in [0,1]$, since here L=1 (b) $\ell \in [0,10]$, since here L=10} 
    \label{fig:decaying_spectra}
\end{figure*}

\subsection{Travelling Plane Wave Data}
We previously considered vibration-like modes with fixed boundaries. This dataset considers sinusoids of uniformly spaced wavenumbers travelling in time and space without explicit boundary conditions, defined by
\begin{equation}
y(\ell,t) = \sum_{n=1}^{N}  \, e^{-\alpha_n t} \sin (k_n \ell + \omega_n t) \; .
\label{eqn:plane_waves}
\end{equation}
This expression resembles a travelling plane wave rather than a vibrational wave. These types of data are commonly found in beamforming problems for RADAR \cite{1449208,haykin1985radar}, ultrasound \cite{jensen1996estimation,YU201643}, audio processing \cite{huang2006acoustic} and other fields using spatial signal processing. In beamforming, the parameter $k_n$ is usually proportional to the direction of arrival. In ultrasound and geophysics, $k_n$ may represent a transverse or longitudinal modes of vibration \cite{jones_1986}. We also have the same assumption of measuring the traveling wave data for a longer length (here 10 times the homogeneous case). Note that time and space are not immediately separable in this equation, as observed in the other datasets. However, \eqref{eqn:plane_waves} can be expressed as 
\begin{align}
    y(\ell,t) &= \sum_{n=1}^{N}  \, e^{-\alpha_n t} \sin (k_n \ell -\phi_n + \phi_n + \omega_n t) \nonumber \\
    &= \sum_{n=1}^{N} \, e^{-\alpha_n t} \left[ \sin (k_n \ell -\phi_n) \cos( \omega_n t + \phi_n) \right. \nonumber \\
    &\qquad\qquad\qquad + \left. \cos (k_n \ell -\phi_n) \sin(\omega_n t + \phi_n ) \right]
    \label{eqn:traveling_decomposed}
\end{align}
for an any arbitrarily chosen $\phi_n$. Hence, a travelling plane wave mode can be expressed as the sum of two standing wave modes, similar to our second dataset, with 90 degree phase differences. In our results, we assume $\phi_n = 0$ so that errors in Table~\ref{table:MSE_dicts} will be small when the extracted components match $\phi_n = 0$ while errors in Table~\ref{tbl:freq_results} will be invariant to $\phi_n$ since the magnitude of the Fourier transform ignores phase delay.

\subsection{Multi-Segment Data}

Our final dataset considers a problem with both amplitude and wavenumber spatial inhomogeneities. This occurs when there are media adjacent to each other. As a signal travels from one medium to the next, the amplitude varies according the the interface's reflection and transmission coefficients and the wavenumber varies according to Snell's law. We generate vibration data in two media connected end-to-end that is continuous and once differentiable at the point where the two different they connect (see \cite{jones_1986} for related theory). We assume the two mediums to have a density ratio of 9 to 1, respectively. The velocity of the wave traveling in the string is inversely proportional to the square root of density, thus the velocities must be in the ratio 3 to 1. The mathematical expression generating the data is:
\begin{eqnarray}
    \label{eqn:segmented_data}
    y(\ell,t) = \begin{cases}
    \sum_{i=1}^{N} \sin(3 k_i \ell) \sin(\omega_i (t/c) ), & 0 \leq \ell < 1 \\
    \sum_{i=1}^{N} \sin( k_i \ell ) (\ell)\sin(\omega_i (t/3c) ), & 1 < \ell \leq 2  \; .
    \end{cases}
\end{eqnarray}
where $c$ is the velocity of propagation in the first medium. The wavenumbers for this problem take a specific solution based on the condition, shown in the Table~\ref{table:params_selected}, assuming fixed, zero-valued boundaries at $\ell = 0$ and $\ell = 2$ and the continuity condition at $\ell = 1$. 


\subsection{Choice of Regularization Parameters}

To provide a fair and informative comparison between each method,
we assume oracle knowledge of $N$ throughout our results. As previously mentioned, deviation from $1$ in the polar problem in wave informed wave factorization can be used as a metric for monitoring convergence and as a stopping condition for the algorithm (and by extension the final choice of $N$), but the appropriate threshold is unclear in the presence of noise and a similar convergence property does not exist for all of our comparison methods. While there are a number of statistical methods for estimating $N$ based on singular value decomposition that could be applied to every method \cite{akaike1998information, Liao2018AnES}, we want our performance to reflect choice of algorithms rather than errors in the estimation of $N$. Hence using oracle knowledge of $N$ to stop the algorithm as appropriate.


For wave-informed matrix factorization, we select $\lambda$ to be $3/4 \sigma(N)$, where $\sigma(N)$ is the N-th singular value of our data matrix $\Y$. This choice of $\lambda$ is intuitively chosen based on the interpretation of low rank factorization as a soft thresholding operation over the singular values \cite{haeffele2014structured, haeffele2019structured}. The $3/4$ factor was obtained empirically based on our examples. However, reducing the factor is reasonable given the additional constraint provided by the Helmholtz equation.

We select $\gamma$ in wave-informed matrix factorization based on our intuition of $1/\sqrt{\gamma}$ being bandwidth of a filter. In general, the bandwidth of exponentially decaying sinusoids discussed here is  small, so $\gamma$ is expected to be relatively large. Since our modes are constrained to a finite length in space, we know the bandwidth of a pure sinusoid will be one sample of the discrete Fourier transform. Therefore, we choose our bandwidth based on this assumption so that
\begin{align}
\label{eqn:gamma_choice}
    \gamma =  \delta \left(\frac{M}{\pi} \right)^2, \;
\end{align}
where $M$ is the number of samples in the spatial domain. While this bandwidth is slightly smaller than some of the data due to spatial variations, this difference should be corrected for in the gradient descent step of wave-informed matrix factorization. An analytical way to motivate the same would be to observe the filter coefficients.
\begin{align*}
    \cfrac{1}{\sqrt{1 + \gamma (k^2 + \boldsymbol{\Lambda}_{ii})^2}}
\end{align*}
is the $i$-th filter coefficient. Assuming pure Dirichlet boundary conditions, we have from \cite{chung2000discrete} that,\begin{align*}
    \boldsymbol{\Lambda}_{ii} = -4 \sin^2 \left( \cfrac{\pi i}{2(M+1)} \right)
\end{align*}
For $M$ large enough compared to $i$, we have that,\begin{align*}
    \boldsymbol{\Lambda}_{ii} = -4 \sin^2 \left( \cfrac{\pi i}{2(M+1)} \right) \approx -\cfrac{\pi^2 i^2}{(M+1)^2}
\end{align*}
Again, for large enough $M$, we can approximate the filter coefficients as below by substituting $\gamma$ from \eqref{eqn:gamma_choice},
\begin{align*}
     \cfrac{1}{\sqrt{1 + \delta \cfrac{M^2}{\pi^2} \left(k^2 -\cfrac{\pi^2 i^2}{(M+1)^2} \right)^2}} \approx \cfrac{1}{\sqrt{1 + \delta   \left(\cfrac{M^2 k^2}{\pi^2} -i^2 \right)^2}}
\end{align*}
where $i$ represents the actual (not angular) wavenumber scale. For instance, when $\delta = 1$, the bandpass Butterworth filter has a -3 dB cut-off wavenumber of 1 m$^{-1}$ in the actual wavenumber scale. For larger values of $\delta$, the -3dB cut-off wavenumber reduces to $1/\delta$ m$^{-1}$ and is useful to pick isolated wavenumbers. The value of $\delta$ used for each case is mentioned in Table~\ref{table:params_selected}. Observe that homogeneous vibrations requires a larger value of $\delta (=100)$  to pick isolated wavenumbers whereas others still work with a value of $\delta=1$.

\section{Results \& Discussions}
\label{sec:results}
For each dataset, we compare our approach with five other non-parametric matrix factorization / blind source separation methods. The algorithms we compare with include independent component analysis (\Dc) \cite{COMON1994287}, dynamic mode decomposition (\E) \cite{tu2014dynamic}, multidimensional empirical mode decomposition (\F) \cite{wu2009ensemble,6144702,1510663}, principal component analysis (\G) \cite{allemang2011application}, and wave-informed K-singular value decomposition (\B) \cite{Tetali2019}. We choose these method because they are widely used in modal analysis throughout the literature or represent a precursor to wave-informed matrix factorization (\Ai). Each method places different assumptions on the extracted modes. Briefly, ICA extracts components components that are statistically independent, PCA extracts components that are orthogonal to each other, DMD extracts components with a set frequency and decay based on estimating a state transition matrix, EMD extracts components by iteratively interpolating between maxima and minima, and WIKSVD extracts components that assume spatial components satisfy the wave equation and are sparse in the frequency domain. By comparison, WIMF extracts spatial components that satisfy the wave equation and minimizes the number of extracted modes. Based on the filter interpretation of WIMF, it has characteristics similar to DMD since the filter's center frequency and bandwidth identifies the frequency and possibly decay of the extracted modes. 

\subsection{General Performance}
In Table~\ref{table:MSE_dicts}, we show the mean squared errors for each of the recovered modes obtained by each method, 
\begin{align}
   \textrm{MSE} = \frac{1}{N} \sum_{n=1}^N \left\| \frac{\d_n}{\| \d_n \|_2} - \frac{ c_n \d_n^{(\textrm{true})}}{\|  \d_n^{(\textrm{true})} \|_2} \right\|^2_2
   \label{eq:mse}
\end{align}
where $\d_n $ is extracted mode and $\d_n^{(\textrm{true})}$ is the true mode. In our analysis, $\d_n^{(\textrm{true})}$ is chosen (without replacement) to be the mode with the maximal correlation (in absolute value) with $\d_n$ and $c_n = \textrm{sign}(\d_n^{\top}\d^{\textrm{(true)}}_n)$. When computing the mean squared error, we normalize each mode since the amplitude is not uniquely determined in matrix decompositions.  

Table~\ref{tbl:freq_results} shows the mean squared error between the Fourier magnitudes true and extracted modes, using the same definition of \eqref{eq:mse} but replacing $\d_n$ and $\d_n^{(\textrm{true})}$ with the magnitude of their Fourier transforms, respectively. These results provide a different perspective, identifying if the methods obtain the correct frequencies ignoring the effect of phase shifts. 

Fig.~\ref{fig:vibrational_modes} further illustrates examples of two modes extracted by each method for each dataset at an SNR of $\infty$. Across all of the datasets, ICA has the overall poorest performance, followed by EMD, PCA, DMD, WIKSVD, and WIMF. We hypothesize ICA performs poorly since wave modes are not random signals, and while the modes are orthogonal, they are not statistically independent. EMD performs poorly since mixtures of sinusoids where the amplitudes and/or frequencies are too similar cannot be recovered by the algorithm \cite{rilling2007one}. DMD and WIKSVD performs most poorly in high noise scenarios and with spatially segmented wavenumbers. DMD performs poorly at high noise because the estimation of the state transition matrices does not consider noise. Though a similar wave-informed regularizer is introduced for WIKSVD, it does not extract the basis as well as WIMF we believe due to the lack of filtering behavior in WIKSVD. This filtering effect particularly improves WIMF over WIKSVD for high noise regimes. The following subsections study the performance with each dataset in greater detail.




\begin{table*}[t]
\setlength{\tabcolsep}{7.5pt}   
\sisetup{scientific-notation=fixed,tight-spacing=true,fixed-exponent=0,round-mode=places,round-precision=1,detect-all=true,table-auto-round,table-format=1.1,table-text-alignment=center}
\label{tbl:all_results}
\caption{\label{table:MSE_dicts} Mean squared errors (times $10^{3}$) between actual basis generating the data for each case (averaged over 20 Monte-Carlo experiments) and the matrix $\D$ obtained from the algorithms (\Ai), (\B), (\Dc), (\E), (\F) and (\G). Highlight values show the best performance.}
\centering
\begin{tabularx}{7in}{@{}l S[]S[] @{\hspace{3.75em}} S[]S[] @{\hspace{3.75em}} S[]S[] @{\hspace{3.75em}} S[]S[] @{\hspace{3.75em}} S[]S[] @{\hspace{3.75em}} S[]S[] @{\hspace{3.75em}} S[]S[]@{}}
\hline
\textbf{Data} & \multicolumn{2}{l}{\centering \textbf{SNR [dB]}} & \multicolumn{2}{l}{\centering \B}  & \multicolumn{2}{l}{\centering \Dc}     & \multicolumn{2}{l}{\centering \E}     & \multicolumn{2}{l}{\centering \F}   & \multicolumn{2}{l}{\G} & \multicolumn{2}{l}{\centering \Ai}    \\ 
\hline
Homogeneous  &-11.84 &  \pm0.38 & 0.77 &    \pm0.13 & 1826.96 &    \pm14.19 & 233.9 &    \pm18.05 & 1173.39 &    \pm159.52 & 61.78 &    \pm8.21  & {\cellcolor{yellow!50}} 0.69 & {\cellcolor{yellow!50}}    \pm0.45 \\ 
Homogeneous   & -5.85 & \pm0.55 & {\cellcolor{yellow!50}} 0.18 & {\cellcolor{yellow!50}} \pm0.06 & 1823.34 & \pm11.24 & 63.94 & \pm7.98 & 1132.57 & \pm130.47 & 13.17 & \pm2.1 & 0.21 & \pm0.08 \\
Homogeneous   & -1.93 & \pm0.31 & {\cellcolor{yellow!50}} 0.08 & {\cellcolor{yellow!50}} \pm0.01 & 1825.9 & \pm11.85 & 26.89 & \pm2.2 & 1206.98 & \pm151.27 & 5.25 & \pm0.65 & 0.14 & \pm0.1  \\
Homogeneous   & 2.28 & \pm0.36 & {\cellcolor{yellow!50}} 0.04& {\cellcolor{yellow!50}} \pm0.01 & 1824.56 & \pm10.09 & 10.5 & \pm1.06 & 1243 & \pm133.39 & 2.03 & \pm0.27 & 0.09 & \pm0.05  \\
Homogeneous   & $\infty$ & & {\cellcolor{yellow!50}} 0.018 & {\cellcolor{yellow!50}} \pm0.001 & 1155.081 & \pm96.941 & 0.172 & \pm0.022 & 1443.895 & \pm160.807 & 0.063 & \pm0.024 & 0.057 & \pm0.009 \\
Non-homogeneous   & -13.09& \pm0.59 & 7.56& \pm1.1 & 1917.6& \pm12.35 & 252.93& \pm26.97 & 1402.98& \pm181.83 & 94.83& \pm60.72 & {\cellcolor{yellow!50}} 3.91& {\cellcolor{yellow!50}} \pm0.57  \\
Non-homogeneous  & -9.41& \pm0.73 & {\cellcolor{yellow!50}} 2.99& {\cellcolor{yellow!50}} \pm0.38 & 1920.05& \pm6.36 & 133.43& \pm17.47 & 1414.42& \pm189.53 & 48.07& \pm48.87 &  3.39& \pm0.45 \\
Non-homogeneous  & -3.38& \pm0.68 & {\cellcolor{yellow!50}} 0.76& {\cellcolor{yellow!50}} \pm0.1 & 1922.28& \pm7.17 & 38.64& \pm5.61 & 1390.92& \pm138.97  & 12.54& \pm19 &  3.12& \pm0.36 \\
Non-homogeneous  &2.69& \pm0.56 &  {\cellcolor{yellow!50}} 0.21& {\cellcolor{yellow!50}}  \pm0.03 & 1922.94& \pm5.0 & 10.42& \pm1.75 & 1381.22& \pm137.57  & 12.59& \pm16.24  & 3.14& \pm0.45 \\
Non-homogeneous   & $\infty$ &   & 0.039& \pm0.005 & 1956.643& \pm67.1 &  {\cellcolor{yellow!50}} 0.000041& {\cellcolor{yellow!50}}  \pm  0.000006 & 1257.568& \pm162.273 & 25.135& \pm33.957 & 3.042& \pm0.37 \\
Travelling Wave  & -11.73& \pm0.627 & {\cellcolor{yellow!50}} 320.05& {\cellcolor{yellow!50}} \pm106.177 & 1913.03& \pm13.517 & 1170.23& \pm57.756 & 1415.61& \pm129.838 & 334.24& \pm82.329 & 1129.04& \pm235.417 \\ 
Travelling Wave  & -5.9& \pm0.543 & 225.47& \pm78.929 & 1927.56& \pm9.552 & 1108.11& \pm32.655 & 1346.61& \pm115.312 & 212.16& \pm89.363 & {\cellcolor{yellow!50}}40.24& {\cellcolor{yellow!50}} \pm1.464   \\
Travelling Wave  & -1.84& \pm0.565 & 199.87& \pm49.975 & 1930.35& \pm13.662 & 1084.06& \pm38.305 & 1267.77& \pm110.204 & 241.56& \pm67.66 & {\cellcolor{yellow!50}}38.29& {\cellcolor{yellow!50}} \pm0.563  \\
Travelling Wave  &1.83& \pm0.834 & 235.93& \pm52.036 & 1932.83& \pm18.775 & 1094.88& \pm29.441 & 1390.48& \pm163.535 & 280.82& \pm62.795 & {\cellcolor{yellow!50}}36.66& {\cellcolor{yellow!50}} \pm0.408 \\
Travelling Wave  &$\infty$ &  & 409.45& \pm13.473 & 1189.03& \pm43.279 & 1115.94& \pm1.251 & 1548.12& \pm75.384 & 342.49& \pm5.63 & {\cellcolor{yellow!50}}35.82& {\cellcolor{yellow!50}} \pm0.271    \\
Segmented & -11.56& \pm0.0004 & 1588.98& \pm38.1574 & 1792.13& \pm25.247  & 1515.05& \pm84.4314 & 1264.24& \pm116.8219 & 1526.52& \pm77.3674 & {\cellcolor{yellow!50}} 891.77& {\cellcolor{yellow!50}} \pm117.3127  \\
Segmented & -9.07& \pm0.0244  & 1594.08& \pm35.9512 & 1794.89& \pm29.2853 & 1532.5& \pm78.7935  & 1260.97& \pm131.089  & 1491.62& \pm60.1776 & {\cellcolor{yellow!50}} 913.24& {\cellcolor{yellow!50}} \pm117.485   \\
Segmented & -3.05& \pm0.0198  & 1638.14& \pm16.267  & 1792.46& \pm27.3304 & 1520.51& \pm68.7968 & 1191.45& \pm127.6302 & 1480.75& \pm34.9955 & {\cellcolor{yellow!50}} 767.29& {\cellcolor{yellow!50}} \pm20.297    \\
Segmented & 5.48& \pm0.0187   & 1664.01& \pm46.2336 & 1786.2& \pm28.6862  & 1483.33& \pm24.9771 & 1158.02& \pm133.9853 & 1470.5& \pm40.8251 & {\cellcolor{yellow!50}} 616.28& {\cellcolor{yellow!50}} \pm4.7015     \\
Segmented &    $\infty$ &              & 889.88& \pm0.001        & 1840& \pm12.7936    & 1006.11& \pm0.001       & 1019.05& \pm0.001      & 1516.11& \pm0.001   & {\cellcolor{yellow!50}} 852.02 &  {\cellcolor{yellow!50}} \pm0.1228       \\
\hline
\end{tabularx}

\vspace{2em}
\setlength{\tabcolsep}{7.5pt}   
\sisetup{scientific-notation=fixed,tight-spacing=true,fixed-exponent=0,round-mode=places,round-precision=1,detect-all=true,table-auto-round,table-format=1.1,table-text-alignment=center }
\caption{\label{tbl:freq_results} Mean squared errors (times $10^3$) between absolute value of Fourier transform of actual basis generating the data for each case (averaged over 20 Monte-Carlo experiments) and the matrix $\D$ obtained from the algorithms (\Ai), (\B), (\Dc), (\E), (\F) and (\G). Highlighted values show the best performance.}
\centering
\begin{tabularx}{7in}{@{}l S[]S[] @{\hspace{3.75em}} S[]S[] @{\hspace{3.75em}} S[]S[] @{\hspace{3.75em}} S[]S[] @{\hspace{3.75em}} S[]S[] @{\hspace{3.75em}} S[]S[] @{\hspace{3.75em}} S[]S[]@{}}
\hline
\textbf{Data} & \multicolumn{2}{l}{\centering \textbf{SNR [dB]}} & \multicolumn{2}{l}{\centering \B}  & \multicolumn{2}{l}{\centering \Dc}     & \multicolumn{2}{l}{\centering \E}     & \multicolumn{2}{l}{\centering \F}   & \multicolumn{2}{l}{\G} & \multicolumn{2}{l}{\centering \Ai}    \\ 
\hline
Homogeneous   & -11.84 & \pm    0.38 & 0.51 & \pm    0.12 & 1731.99 & \pm    23.89 & 229.51 & \pm    18.19 & 891.27 & \pm    151.17   & 59.92 & \pm    7.87  & {\cellcolor{yellow!50}} 0.23 & {\cellcolor{yellow!50}} \pm    0.14 \\ 
Homogeneous   &-5.85 & \pm    0.55 & 0.11 & \pm    0.05 & 1719.27 & \pm    21.25 & 62.51 & \pm    7.84 & 838.18 & \pm    150.21  & 12.6 & \pm    2.05  & {\cellcolor{yellow!50}} 0.06 & {\cellcolor{yellow!50}} \pm    0.03 \\
Homogeneous   & -1.93 & \pm    0.31 & 0.04 & \pm    0.01 & 1725.65 & \pm    24.87 & 26.12 & \pm    2.17 & 733.01 & \pm    104.9 & 4.88 & \pm    0.59  & {\cellcolor{yellow!50}} 0.04 & {\cellcolor{yellow!50}} \pm    0.02  \\
Homogeneous   &2.28 & \pm    0.36 & {\cellcolor{yellow!50}} 0.02 & {\cellcolor{yellow!50}} \pm    0.01 & 1723.94 & \pm    25.25 & 10.13 & \pm    1.02 & 768.8 & \pm    86.18  & 1.83 & \pm    0.21  & 0.03 & \pm    0.01 \\
Homogeneous   & $\infty$ & & 0.011 & \pm    0.001 & 782.744 & \pm    110.157 & 0.152 & \pm    0.019 & 1020.455 & \pm    111.294 & 0.016 & \pm    0.006 & {\cellcolor{yellow!50}} 0.017 & {\cellcolor{yellow!50}} \pm    0.002 \\
Non-homogeneous   & -13.09& \pm   0.59 & 4.25& \pm   0.63 & 1767.38& \pm   20.51 & 238.68& \pm   24.97 & 958.78& \pm   122.03   & 78.86& \pm   52.51 & {\cellcolor{yellow!50}} 1.8& {\cellcolor{yellow!50}} \pm   0.42 \\
Non-homogeneous  & -9.41& \pm   0.73 & {\cellcolor{yellow!50}} 1.5& {\cellcolor{yellow!50}} \pm   0.19 & 1771.97& \pm   15.62 & 125.96& \pm   16.54 & 944.35& \pm   135.65 & 38.37& \pm   39.95 & 1.47& \pm   0.31  \\
Non-homogeneous  & -3.38& \pm   0.68 & {\cellcolor{yellow!50}} 0.32& {\cellcolor{yellow!50}} \pm   0.04 & 1779.36& \pm   15.44 & 35.97& \pm   5.36 & 918.64& \pm   158.62 & 8.67& \pm   14.17 & 1.32& \pm   0.26  \\
Non-homogeneous  & 2.69& \pm   0.56 & {\cellcolor{yellow!50}} 0.08& {\cellcolor{yellow!50}} \pm   0.01 & 1771.68& \pm   15.31 & 9.45& \pm   1.68 & 869.36& \pm   151.3 & 8.36& \pm   12.72 & 1.34& \pm   0.3 \\
Non-homogeneous   & $\infty$ & & {\cellcolor{yellow!50}}	0.017& {\cellcolor{yellow!50}} \pm   0.002 &	1753.32& \pm   79.811 & 0.000006& \pm   0.000001 & 948.256& \pm   116.393 &	1.274& \pm   0.254  & 18.944& \pm   28.261 	\\
Travelling  Wave  &-11.73& \pm   0.627 & {\cellcolor{yellow!50}} 102.57& {\cellcolor{yellow!50}} \pm   62.53   & 1878.85& \pm   14.601 & 668.07& \pm   23.091  & 1113.04& \pm   74.405  & 147.76& \pm   38.868 & 628.26& \pm   187.239  \\ 
Travelling  Wave  & -5.9& \pm   0.543 & {\cellcolor{yellow!50}} 8.91& {\cellcolor{yellow!50}} \pm   1.364    & 1893.01& \pm   8.316 & 597.02& \pm   2.652   & 1040.68& \pm   64.72  & 26.04& \pm   4.451  & 33.99& \pm   0.81   \\
Travelling  Wave  & -1.84& \pm   0.565 & {\cellcolor{yellow!50}} 5.1& {\cellcolor{yellow!50}} \pm   0.825     & 1899.45& \pm   12.028 & 588.02& \pm   1.554   & 970.5& \pm   63.797   & 10.13& \pm   2.314  & 32.92& \pm   0.558  \\
Travelling  Wave  &1.83& \pm    0.834 & {\cellcolor{yellow!50}} 3.95& {\cellcolor{yellow!50}} \pm    1.065    & 1902.72& \pm    16.441 & 585.46& \pm    0.532   & 1088.83& \pm    150.451 & 4.08& \pm    0.976 & 31.6& \pm    0.329     \\
Travelling  Wave &  $\infty$ & & 1.47& \pm   0.2      & 1081.02& \pm   40.082 & 585& \pm   0         & 1446.47& \pm   60.782 &  {\cellcolor{yellow!50}}  0.12&  {\cellcolor{yellow!50}}  \pm   0.108  &  31.32&  \pm   0.295  \\
Segmented & -11.56& \pm0.0004 & 1037.52& \pm36.026  & 1373.83& \pm24.3314 & 949.2& \pm33.3454  & 724.26& \pm80.7343 & 957.85& \pm35.2736 & {\cellcolor{yellow!50}} 536.7& {\cellcolor{yellow!50}} \pm71.3607    \\
Segmented & -9.07& \pm0.0244  & 1067.23& \pm47.922  & 1383.48& \pm25.9373 & 907.05& \pm27.7681 & 626.9& \pm88.1976  & 938.54& \pm34.8523 & {\cellcolor{yellow!50}} 502.69& {\cellcolor{yellow!50}} \pm98.836    \\
Segmented & -3.05& \pm0.0198  & 1249.49& \pm79.8907 & 1382.94& \pm25.0862 & 863.72& \pm45.2521 & 496.23& \pm39.492  & 923.74& \pm26.4544 & {\cellcolor{yellow!50}} 397.71& {\cellcolor{yellow!50}} \pm128.3619  \\
Segmented & 5.48& \pm0.0187   & 1175.7& \pm18.6612  & 1372.06& \pm24.5076 & 973.45& \pm37.9708 & 448.47& \pm63.7549 & 908.77& \pm21.7454 & {\cellcolor{yellow!50}} 428.47& {\cellcolor{yellow!50}} \pm6.6073    \\
Segmented & $\infty$     &   & 623.45 & \pm 0.001       & 1337.82& \pm 25.0305 & {\cellcolor{yellow!50}} 348.7& {\cellcolor{yellow!50}} \pm 0.001        & 426.55& \pm0.001   & 887.16& \pm0.001   & 357.87& \pm0.1348            \\ 
\hline
\end{tabularx}
\end{table*}

\subsection{Homogeneous Vibration Data With Fixed Boundaries}
\label{subsec:homog_vibration}

Table~\ref{table:MSE_dicts} and \ref{tbl:freq_results} show that this relatively simple dataset achieves the overall best performance among the datasets. Wave-informed matrix factorization, wave-informed K-SVD, and dynamic mode decomposition perform best, albeit under different conditions. Wave-informed K-SVD and dynamic mode decomposition perform the best in no-noise conditions. Yet, dynamic mode decomposition is more sensitive to noise. This is evident in Tables~\ref{table:MSE_dicts},\ref{tbl:freq_results} as with increasing noise the errors indicated in the tables increase.  Wave-informed K-SVD performs best in low to mild noise ($\leq-3$~dB SNR) whilst wave-informed matrix factorization performs best in heavy noise ($>-3$~dB SNR). 


\subsection{Inhomogeneous Vibration Data With Spatial Decay}
\label{subsec:inhomog_vibration}

These results follow a similar trend as the homogeneous vibration data. In the inhomogeneous data, the spatial decay spreads the signal across the wavenumber domain, resulting in generally less compressible data that is more sensitive to noise. As a result, we see a reduction in performance for WIMF, WIKSVD, DMD in Table~\ref{table:MSE_dicts} and \ref{tbl:freq_results}. In case of infinite SNR, DMD performs best as it is designed to extract damped sinusoids \cite{noack2016recursive}. However, it remains sensitive to noise. 

PCA and EMD appear unstable with this dataset. This occurs because the parameter $\beta_n$ varies across each method and SNR. This randomness in conjunction with the incorrect extraction of multiple wavenumbers in a single component, as previously described for the homogeneous data, is the source for this error. As a result, EMD and PCA are to be sensitive to the variation in spatial decay. 


\begin{figure*}[p!]

    \centering
    \includegraphics[]{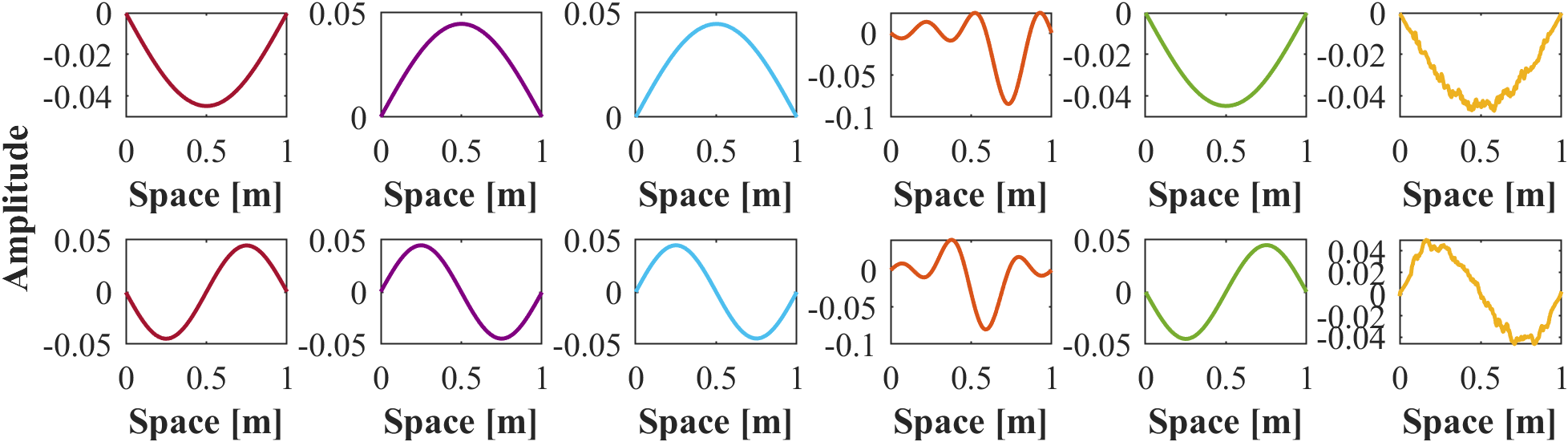}
    \caption{(a) Homogeneous Vibration}

    \centering
    \includegraphics[]{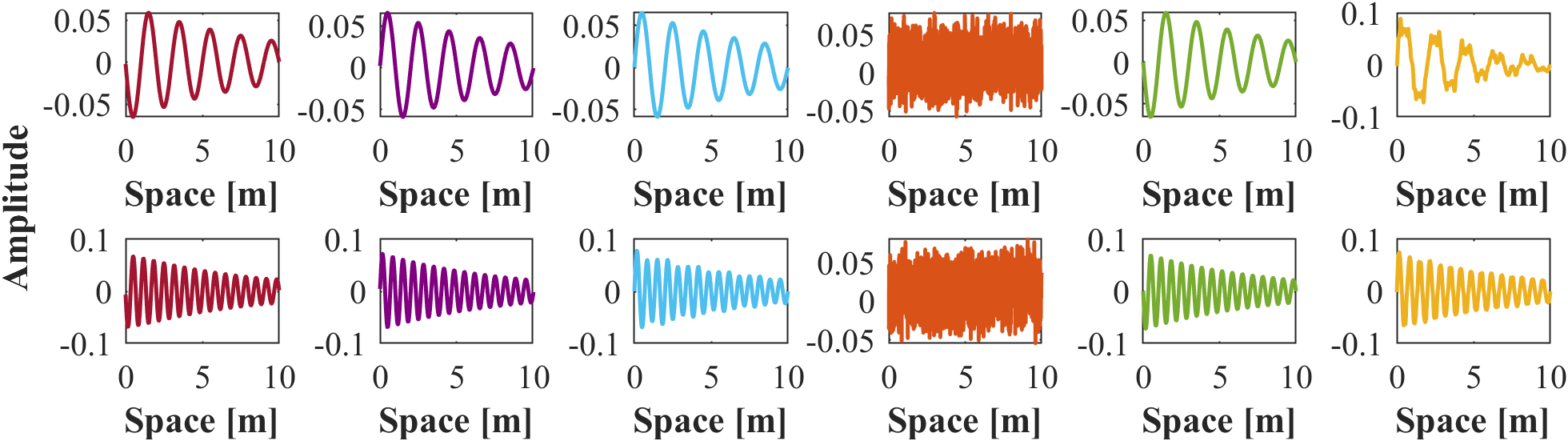}
    \caption*{(b) In-homogeneous Vibration}
  

    \centering
    \includegraphics[]{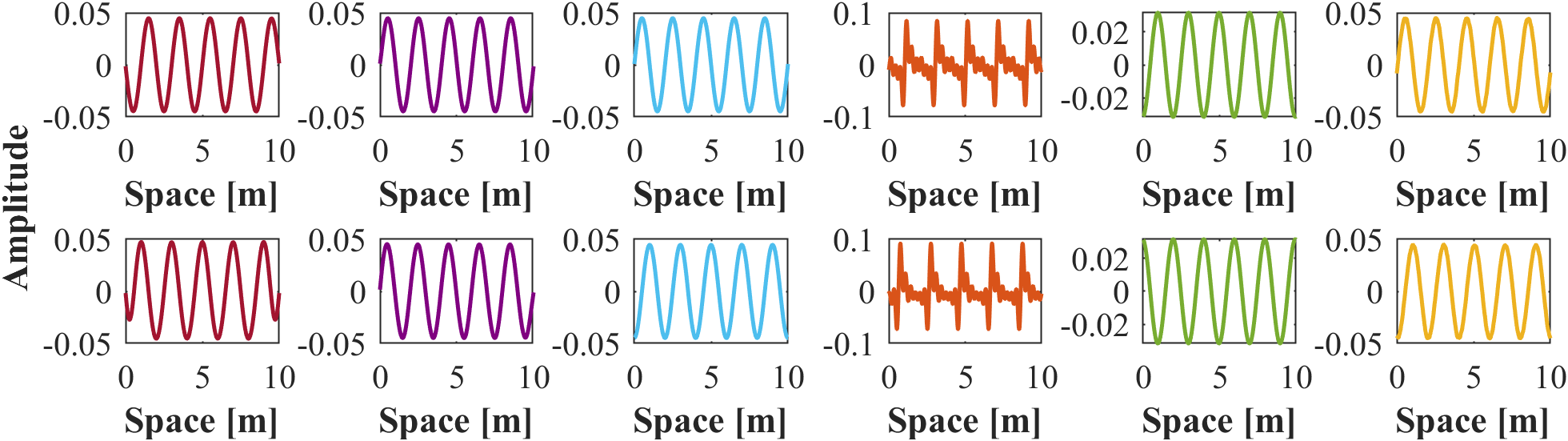}
    \caption*{(c) Traveling Plane Waves}

    \centering
    \includegraphics[]{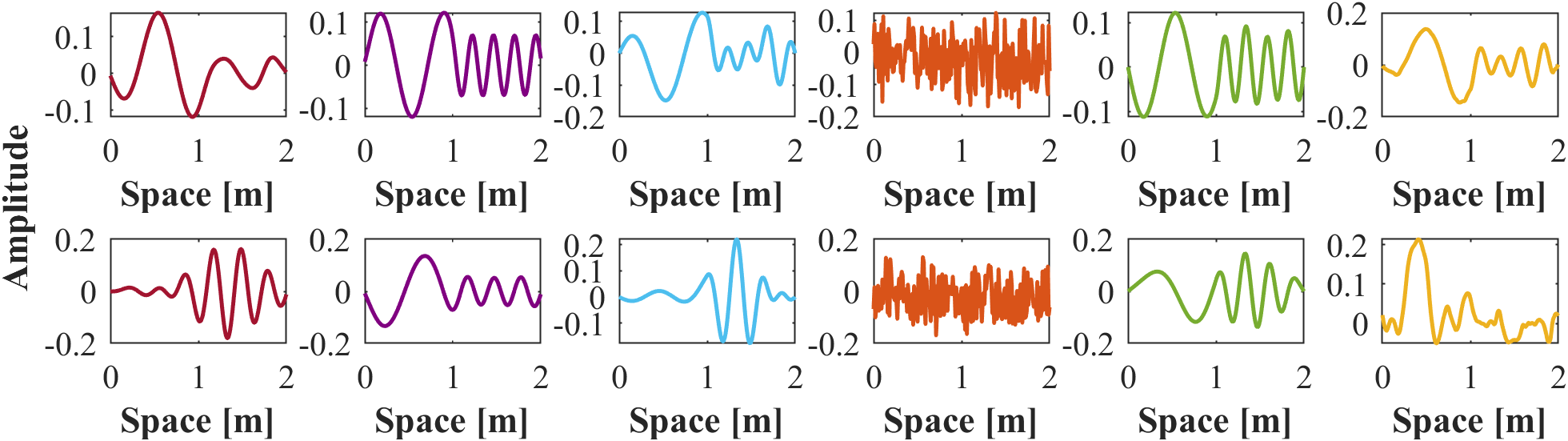}
    \caption*{(d) Multiple Segments}
    \caption{Two recovered modes (rows) obtained from (\Ai), (\B), (\C), (\Dc), (\E), (\F) for SNR = $\infty$.\label{fig:vibrational_modes}}    
    
\end{figure*}

\begin{figure*}[ht!]
    \centering
    \includegraphics[width=0.49\textwidth]{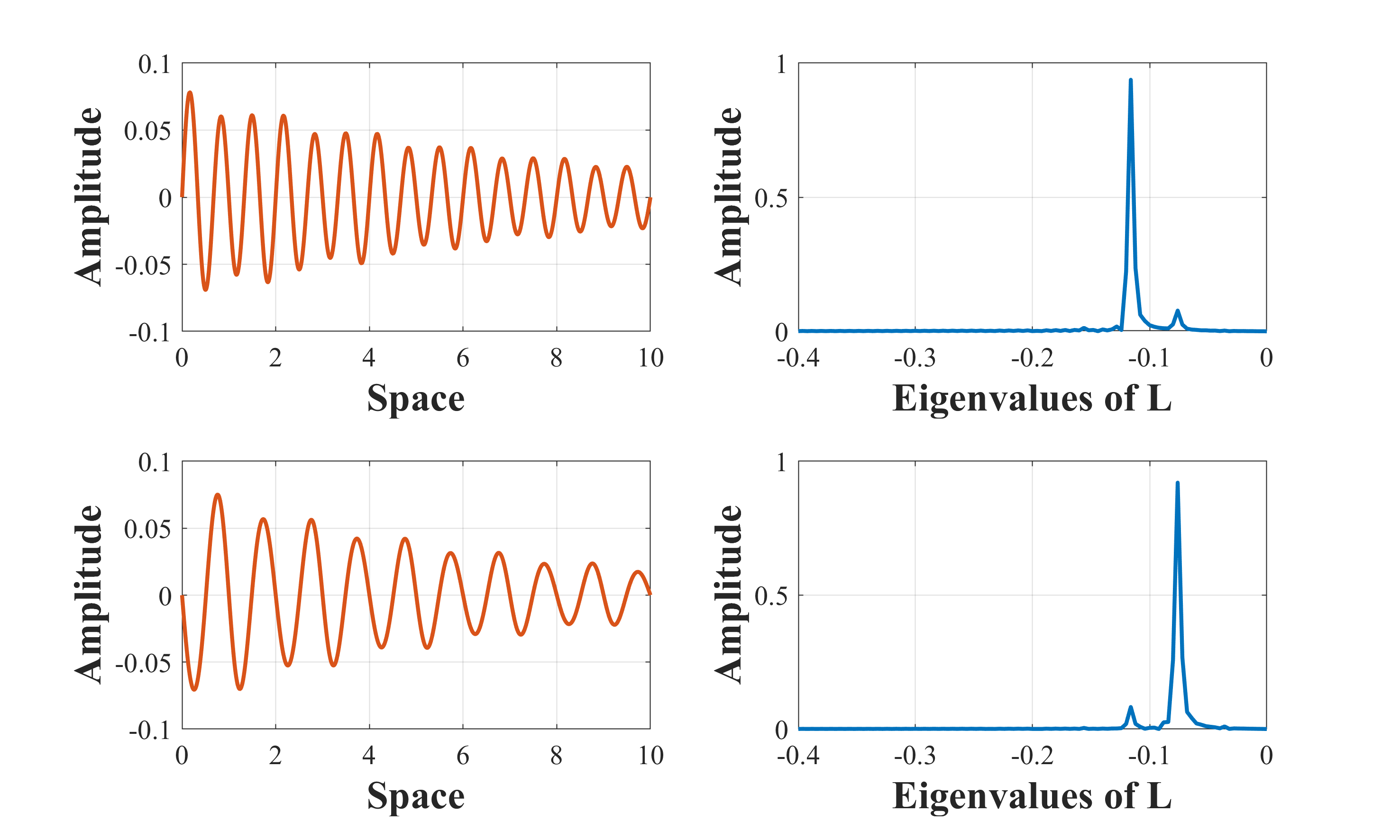}
    \includegraphics[width=0.49\textwidth]{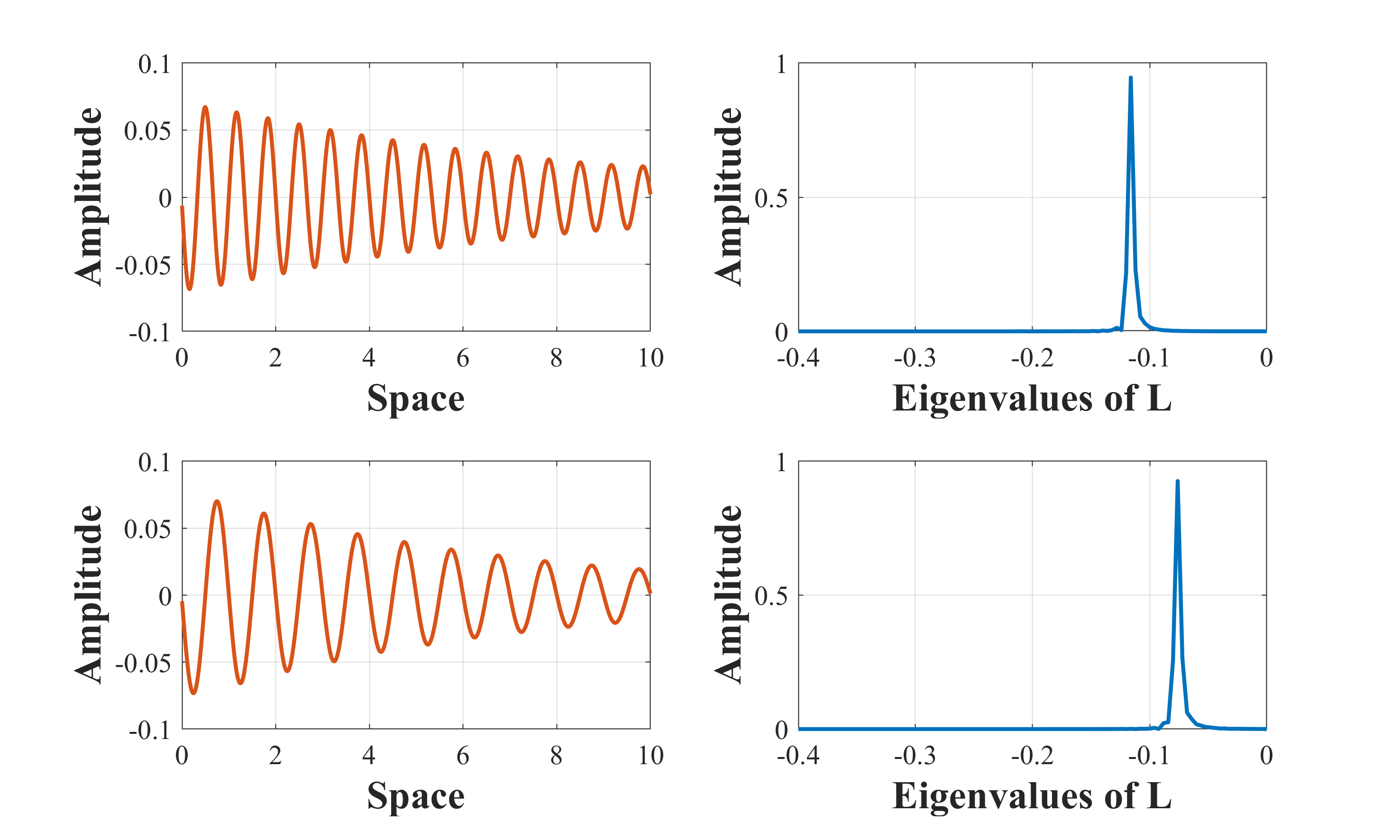}
    \caption{(a) Two columns from the basis obtained from PCA and the absolute value in the transformed domain (transformed to the eigenbasis of $\L$) (b) Two columns from the basis obtained from wave-informed matrix factorization and the absolute value in the transformed domain}
    \label{fig:x_k}
\end{figure*}

\subsection{Traveling Plane Wave Data}

\label{subsec:traveling_waves_data}

Due to the expression derived in \eqref{eqn:traveling_decomposed}, we extract $2N$ modes rather than $N$ modes. This is seen in Fig.~\ref{fig:vibrational_modes}(c), where the two most dominant modes have the same frequencies. It is also important to note that the decomposition is not unique due to the arbitrary phase term $\phi_n$. This discrepancy seen in the difference in PCA's relatively poor performance in Table~\ref{table:MSE_dicts} (where $\phi_n$ is assumed to be $0$) and the excellent performance in Table~\ref{tbl:freq_results}, which ignores the phase in its computation. We further observe in Fig.~\ref{fig:vibrational_modes}(c) that PCA does not have precise phases of $0$ and $90$ degrees. Similarly, DMD incorrectly extracts modes with equal phases. This is because DMD extracts a complex-valued factorization. This reconstructs the data correctly but creates a large error when compared with the true real-valued modes. Note that the Fig.~\ref{fig:vibrational_modes}(c) shows the real part of the DMD result but the result in Table~\ref{table:MSE_dicts} and \ref{tbl:freq_results} compute the error with the complex values. In no other datasets does DMD produce a strong imaginary component. 

In addition, WIKSVD appears to have difficulty extracting a mode with a $90$~degree phase based on the errors in Table~\ref{table:MSE_dicts} and that the second most prominent mode in Fig.~\ref{fig:vibrational_modes}(c) has the wrong frequency. This can be explained due to our choice of Laplacian matrix $\L$, which inherently assumes fixed, zero-valued boundary conditions. Similarly, we observe that WIMF in Fig.~\ref{fig:vibrational_modes}(c) forces zero-valued boundaries, but the second mode immediately adjusts to have an effective 90 degree phase. This shows that WIMF is more adaptable than WIKSVD. As a result, WIMF overall achieves the best results, but has a larger error relative to the previous two datasets.

\subsection{Multi-Segment Data} 
\label{subsec:multi_segment_data}
Tables~\ref{table:MSE_dicts} and \ref{tbl:freq_results} show that this dataset is the most difficult of the four. Under low noise conditions, WIMF achieves the best performance. From  Fig.~\ref{fig:vibrational_modes}(d), we observe the relatively large error is due to wavenumbers information "leaking" from one segment to the other. In particular, WIKSVD and DMD include the the modal bevhavior both regions in their components. Hence, these solution incorrectly include two different wavenumbers in their components. 

PCA and EMD achieve decent localization of the modes but fail to extract the correct frequency content, as evident in Table~ \ref{tbl:freq_results}. WIMF, EMD, and PCA all demonstrate some degree of separation and localization, but are not tightly bound to their spatial region. In particular, WIMF does not achieve stronger localization because the wave equation constraints are placed on each mode rather than the complete solution, and the discontinuity between segments does not satisfy the wave equation for each individual mode.

\subsection{General discussion on wave-informed matrix factorization}

In the performance analysis of this method, we observe that \emph{wave-informed matrix factorization} successfully recovers pure sinusoids with zero phase in the homogeneous vibrations case in Section~\ref{subsec:homog_vibration} and recovers decaying sinusoids in the inhomogenous vibrations case in Section~\ref{subsec:inhomog_vibration}. In Section~\ref{subsec:traveling_waves_data}, it recovers approximate cosines that start at zero and eventually converge to the sinusoid shape, when the true result has a phase shift. For the multi-segment case in Section~\ref{subsec:multi_segment_data}, it recovers sinusoidal vibrations that locally to each region. 

We would like to emphasize here that hard constraining the wave equation, i.e. solving an optimization similar to
\begin{align*}
    \underset{\D,\X,N,\k}{\argmin} \frac{1}{2} \| \Y - \D\X^{\top} \|_F^2 + \frac{\lambda}{2}\left( \|\D\|_F^2 + \|\X\|_F^2 \right)  \\ \st \L \D_i = -k_i^2 \D_i, \; \forall i \in [N]
\end{align*}
would be equivalent to solving, for $\D$, the eigenvectors of $\L$ that closely span (in the $\ell_2$ sense) the column space of $\Y$. These are pure sinusoids with zero phase. In contrast to this, we observe that the columns of $\D$ are dictated by both data and the soft constraint introduced by the wave equation. In particular, we note that the columns extracted in most of the cases only approximately satisfy the wave-constraint and do not satisfy the wave equation. This property distinguishes our algorithm from traditional factorization methods that are either data dependent or are entirely parameterized by the assumed physics. We attribute this to the recovery of components such as exponentially decaying sinusoids (which do not satisfy the wave equation).

\section{Conclusions}
We have introduced  wave-informed matrix factorization and developed a framework and an algorithm with provable, global optimally guarantees for the same. More generally, this work introduced a methodology to enforce linear homogeneous partial differential equation to influence a matrix factorization algorithm, and  demonstrated this for the case of the time-independent version of the Helmholtz wave equation. The output from the algorithm was compared with that of state-of-the-art algorithms for modal and component analysis. We demonstrated that the wave-informed approach learns representations that are more physically relevant and practical for the purpose of modal analysis. 

Future work will include generalizing this approach to a variety of linear PDEs beyond the wave equation (especially to adapt the works \cite{lai2020full, osti_1487358} to our framework) as well as wave propagation along more than one dimension and extension to applications in baseline-free anomaly detection for structural health monitoring \cite{alguri2018baseline, alguri2021sim}.

\section{Acknowledgements}

This work is partially supported by NSF EECS-1839704, NSF CISE-1747783, NIH NIA 1R01AG067396, ARO MURI
W911NF-17-1-0304, and NSF-Simons MoDL 2031985.

\ifCLASSOPTIONcaptionsoff
  \newpage
\fi



%

\IEEEpeerreviewmaketitle

\bibliographystyle{IEEEtran}
\bibliography{bibliography.bib}

\begin{IEEEbiography}[{\includegraphics[width=1in,height=1.25in,clip,keepaspectratio]{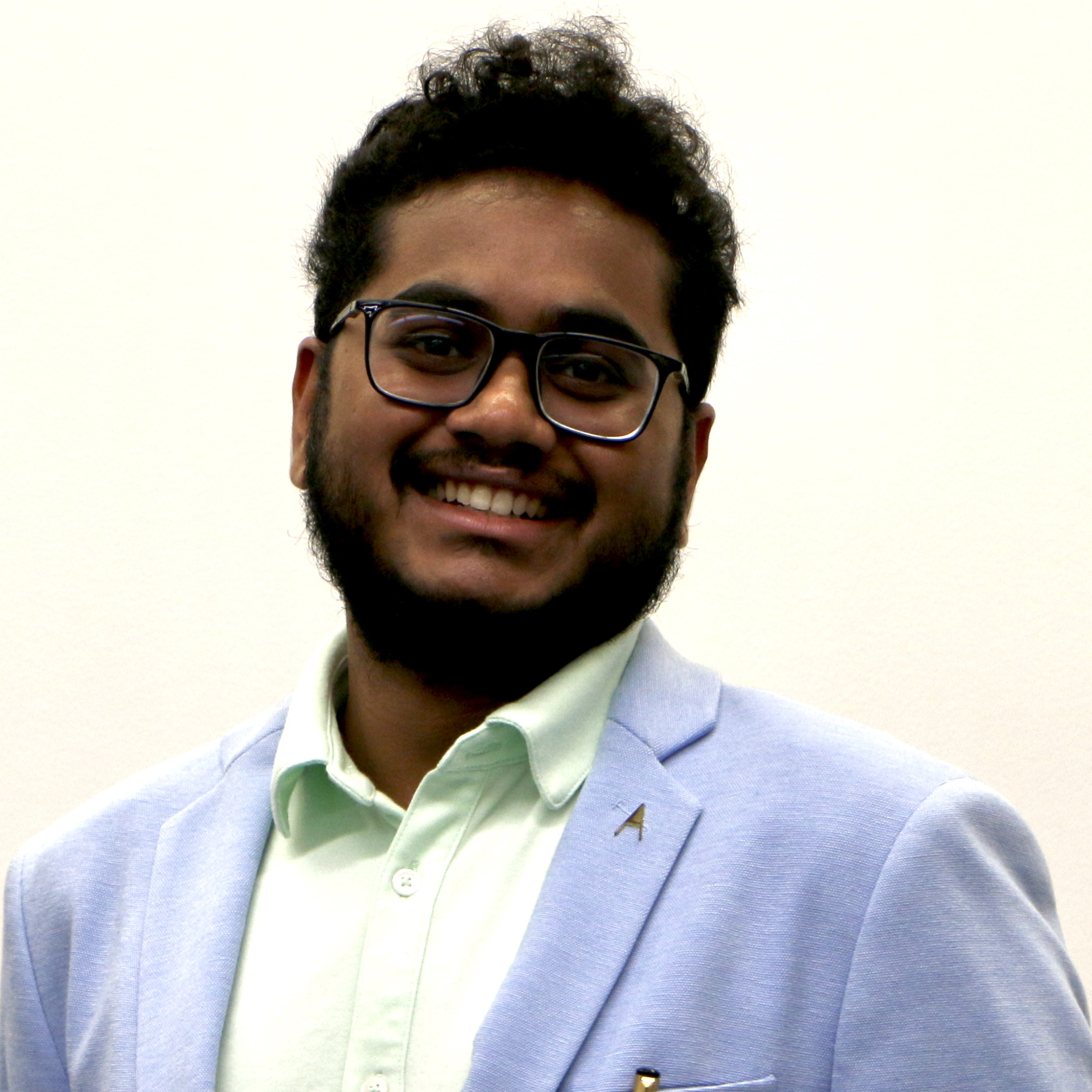}}]{Harsha Vardhan Tetali}
has a PhD in Electrical and Computer Engineering from University of Florida, Gainesville. Prior to this, he received his B.Tech in Electronics and Communication Engineering from SVNIT, Surat and M.Tech in Electrical Engineering from IIT Gandhinagar. His research spans signal processing, machine learning, and optimization theory. He currently works as a staff engineer at Marvell Technology.
\end{IEEEbiography}

\begin{IEEEbiography}[{\includegraphics[width=1in,height=1.25in,clip,keepaspectratio]{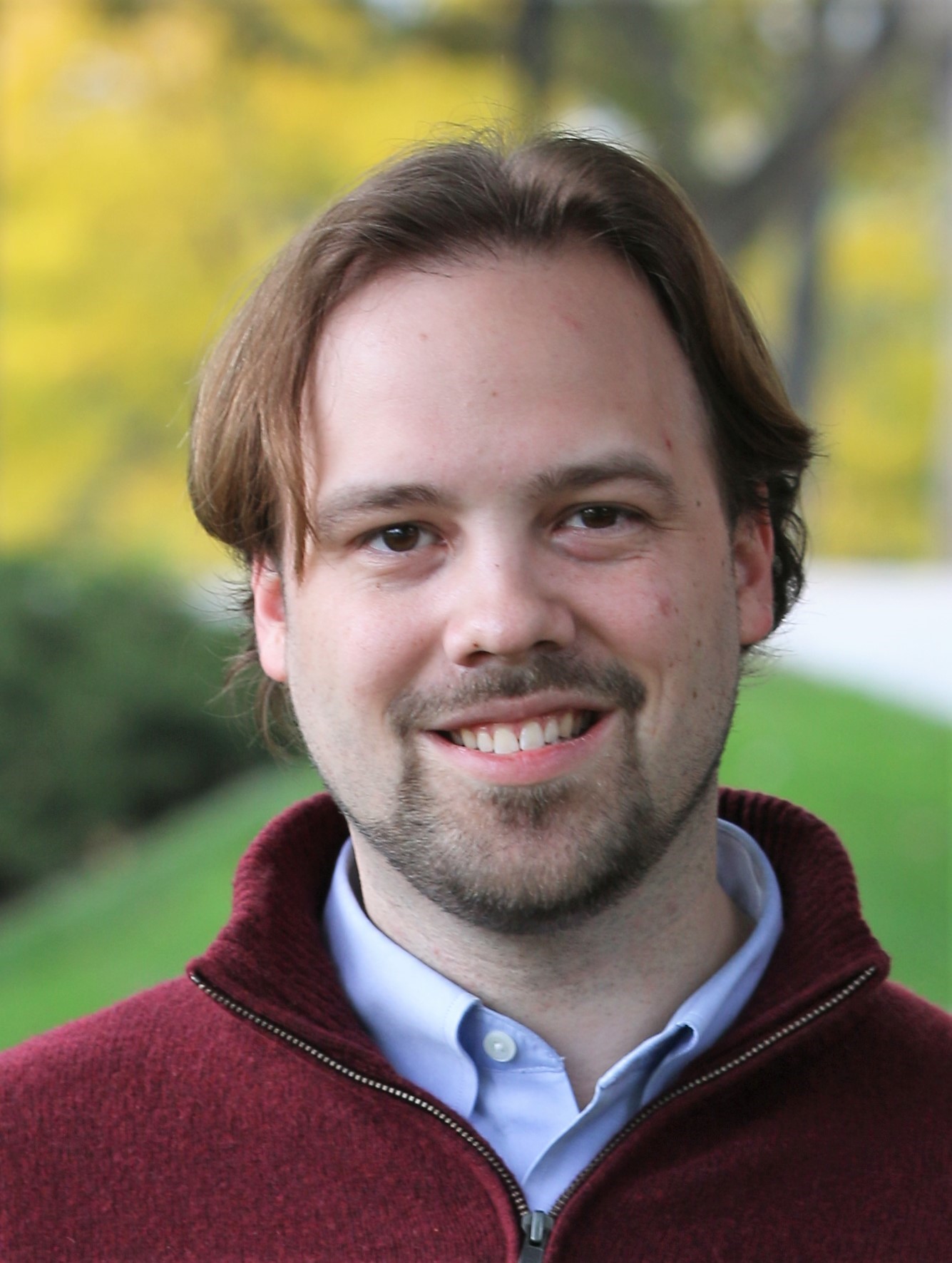}}]{Joel B. Harley} (S'05-M'14) received his B.S. degree in Electrical Engineering from Tufts University in Medford, MA, USA. He received his M.S. and Ph.D. degrees in Electrical and Computer Engineering from Carnegie Mellon University in Pittsburgh, PA, USA in 2011 and 2014, respectively.
In 2018, he joined the University of Florida, where he is currently an associate professor in the Department of Electrical and Computer Engineering. Previously, he was an assistant professor in the Department of Electrical and Computer Engineering at the University of Utah. His research interests include integrating novel signal processing, machine learning, and data science methods for the analysis of waves and time-series data.
Dr. Harley's awards and honors include the 2021 Achenbach Medal from the International Workshop on Structural Health Monitoring, a 2020 IEEE Ultrasonics, Ferroelectrics, and Frequency Control Society Star Ambassador Award, a 2020 and 2018 Air Force Summer Faculty Fellowship, a 2017 Air Force Young Investigator Award, a 2014 Carnegie Mellon A. G. Jordan Award (for academic excellence and exceptional service to the community). He has published more than 90 technical journal and conference papers, including four best student papers. He is a member of the Ultrasonics, Ferroelectrics, and Frequency Control Society, a member of the IEEE Signal Processing Society, and a member of the Acoustical Society of America.
\end{IEEEbiography}

\begin{IEEEbiography}[{\includegraphics[width=1in,height=1.25in,clip,keepaspectratio]{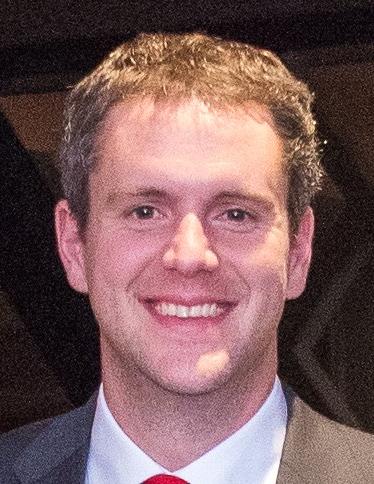}}]{Benjamin D. Haeffele}
is an Associate Research Scientist in the Mathematical Institute for Data Science and the Center for Imaging Science at Johns Hopkins University. His research interests involve developing theory and algorithms for processing high-dimensional data at the intersection of machine learning, optimization, and computer vision. In addition to basic research in data science he also works on a variety of applications in medicine, microscopy, and computational imaging. He received his Ph.D. in Biomedical Engineering at Johns Hopkins University in 2015 and his B.S. in Electrical Engineering from the Georgia Institute of Technology in 2006.
\end{IEEEbiography}

\clearpage

\input{supplement}

\end{document}

%% file: bdhmath.tex
\usepackage{amsmath,amsfonts,amssymb}
\usepackage{amsthm}

\def \st {\ \ \textnormal{s.t.} \ \ }
\def \ST {\ \ \textnormal{subject to} \ \ }
\def \Re{\mathbb{R}}

\DeclareMathOperator*{\argmin}{arg\,min}
\DeclareMathOperator*{\argmax}{arg\,max}

\theoremstyle{plain}
\usepackage{mathabx}

\newtheorem{corollary}{Corollary}

\newtheorem{lemma}{Lemma}
\newtheorem{proposition}{Proposition}
\newtheorem{theorem}{Theorem}
\newtheorem*{theorem*}{Theorem}
\newtheorem*{corollary*}{Corollary}

\theoremstyle{remark}

\theoremstyle{definition}

%% file: supplement.tex
\setcounter{page}{1}
\section{Supplementary Material}

\subsection{Global Optimal Solution is Achievable}
\label{ap:global_optimal}
{
\begin{proposition}
The optimization problem in \eqref{eq:main_obj} is a special case of the problem considered in \cite{haeffele2019structured}, which shows that a global minimum is achievable.
\end{proposition}
}

\begin{proof}
The problem considered in \cite{haeffele2019structured} is stated in \eqref{eq:gen_obj}. Comparing it with (\ref{eq:main_obj}) we have:
\begin{gather}
    \ell \left( \D\X^{\top} \right) = \tfrac{1}{2}\| \Y - \D\X^{\top} \|_F^2\\
    \widebar{\theta} \left( \d_i, \x_i \right) = \tfrac{1}{2} \! \left( \|\x_i\|_F^2 + \min_{k_i} \d_i^\top \A(k_i) \d_i \right)
\end{gather}
Observe that $\ell ( \widehat{\Y} ) = \tfrac{1}{2} \| \Y - \widehat{\Y} \|^2_F$ is convex and differentiable w.r.t. $\widehat \Y$.  
 To realize that the optimization problem in \eqref{eq:main_obj} is a special case of the problem considered in \cite{haeffele2019structured}, it suffices to check that $\widebar \theta(\d,\x)$ satisfies three conditions:
\begin{enumerate}
    \item $\widebar \theta(\alpha \d, \alpha \x) = \alpha^2 \theta(\d, \x), \ \forall (\d,\x)$ and $\forall \alpha \geq 0$.
    
    For any $\alpha > 0$, $\forall (\d, \x)$ :
    \begin{align}
        \widebar{\theta}\left( \alpha \d, \alpha \x \right)  &= \| \alpha \d \|^2_2 + \| \alpha \x \|^2_2 + \gamma \min_{k} \| \alpha \L \d + k^2 \alpha \d \|^2_F \nonumber \\
        &= \alpha^2 \| \d \|^2_2 + \alpha^2 \| \x \|^2_2 + \alpha^2 \gamma \min_{k} \| \L \d + k^2 \d \|^2_F \nonumber \\
       &= \alpha^2 \widebar{\theta}(\d,\x)
    \end{align}
    
    where we note that scaling $\d$ by $\alpha > 0$ does not change the optimal value of $k$ in the third term, allowing $\alpha^2$ to be moved outside of the norm.
    
    \item $\widebar \theta(\d, \x) \geq 0, \ \forall (\d,\x)$.

    All terms in $\widebar{\theta}(\d,\x)$ are non-negative, thus, $\forall (\d,\x)$, we have $\widebar{\theta}(\d,\x)\geq 0$.
    \item For all sequences $(\d^{(n)},\x^{(n)})$ such that $\|\d^{(n)}(\x^{(n)})^\top\| \rightarrow \infty$ then $\widebar \theta(\d^{(n)},\x^{(n)}) \rightarrow \infty$.
    
    Here, note that the following is true for all $(\d,\x)$:
    \begin{equation}
    \begin{split}
            \|\d \x^\top \|_F = \|\d\|_2 \|\x\|_2 \leq \tfrac{1}{2}(\|\d\|_2^2 + \|\x\|_2^2) \\ \leq \tfrac{1}{2}\left(\|\d\|_2^2 + \|\x\|_2^2 + \min_{k} \| \L \d + k^2 \d \|^2_2\right)
    \end{split}
    \end{equation}
    As a result we have $\forall (\d, \x)$ that $\|\d\x^\top\|_F \leq \widebar \theta(\d,\x)$, completing the result.
\end{enumerate}
Hence all of the conditions are satisfied for the global optimal solution to be achievable.
\end{proof}

\subsection{Solution to the Polar Problem }
\label{ap:polar_problem}

We restate theorem \ref{thm:polar} below.
{
\renewcommand{\thetheorem}{\ref{thm:polar}}
\begin{theorem}
For \eqref{eq:main_obj}, the polar problem in \eqref{eq:polar_def} is 
\begin{align}
\Omega_\theta^\circ (\Z) = & \max_{\d, \x, k} \, \d^\top \Z \x  \nonumber \\
&\mathrm{s.t.} ~ \d^\top \A(k) \d \leq 1, \|\x\|^2_F \leq 1, \ 0 \leq k \leq 2. \nonumber
\end{align}
where $\A(k)$ is as defined in \eqref{eq:Ak}.
Further, if we define $k^*$ as%
\begin{equation}
k^* = \argmax_{k \in [0,2]} \| \A(k)^{-1/2} \Z \|_2 \; . 
\end{equation}
Then the optimal values of $\d,\x, k$ are given as $\d^* = \A(k^*)^{-1/2} \widebar \d$, $\x^* = \widebar \x$, and $k^*$.  Where $\widebar \d$ and $\widebar \x$ are the left and right singular vectors, respectively, associated with the largest singular value of $\A(k^*)^{-1/2} \Z$. 
\end{theorem}
\addtocounter{theorem}{-1}
}


\begin{proof}
The polar problem associated with the objective of the form \eqref{eq:gen_obj} as given in \cite{haeffele2019structured} is:
\begin{gather}
       \Omega_\theta^\circ (\Z) = \sup_{\d,\x} \d^{\top} \Z \x \st \widebar \theta(\d,\x) \leq 1 
\end{gather}
For our particular problem, due to the bilinearity between $\d$ and $\x$ in the objective the above is equivalent to:
\begin{gather}
\label{eqn:polar_descrip}
       \begin{split}
         \Omega_\theta^\circ (\Z)   &= \sup_{\d,\x} \d^{\top} \Z \x \\ \st \| \x \|^2_2 &\leq 1, \|\d\|^2_2 + \gamma \min_{k} \| \L \d + k^2 \d \|^2_2 \leq 1
       \end{split}
\end{gather}
Note that this is equivalent to moving the minimization w.r.t. $k$ in the regularization constraint to a maximization over $k$:
 \begin{gather}
\begin{split}
             \Omega_\theta^\circ (\Z)   &= \sup_{\d,\x,k} \d^{\top} \Z \x \\ \st \| \x \|^2_2 &\leq 1, \|\d\|^2_2 + \gamma \| \L \d + k^2 \d \|^2_2 \leq 1
\end{split}
\end{gather}
Next, note that maximizing w.r.t. $\d$ while holding $\x$ and $k$ fixed is equivalent to solving a problem of the form:
\begin{align}
\max_\d \langle \d, \Z \x \rangle \st \d^\top \A(k_i) \d \leq 1
\end{align}
for some positive definite matrix $\A(k_i)$.  If we make the change of variables $\widebar \d = \A(k_i)^{1/2} \d$, this then becomes:
\begin{align}
\max_{\widebar \d} \ \ &\langle \widebar \d, \A(k_i)^{-1/2} \Z \x \rangle \st \| \widebar \d \|_2^2 \leq 1 = \| \A(k_i)^{-1/2} \Z \x \|_2 
\end{align}
where the optimal $\widebar \d$ and $\d$ are obtained at 
\begin{align}
\widebar \d_{opt} &= \frac{\A^{-1/2} \Z \x}{\|\A^{-1/2} \Z \x\|_2} \\
\label{eq:d_opt}
\d_{opt} &= \A^{-1/2} \widebar \d_{opt} = \frac{\A^{-1} \Z \x}{\|\A^{-1/2} \Z \x\|_2}
\end{align}
For our particular problem, if we make the change of variables $\widebar k = k^2$ we have that $\A$ is given by:
\begin{equation}
\A(\widebar k) = (1+\gamma \widebar k^2) \I + \gamma \L^2 + 2 \widebar k \gamma \L
\end{equation}
where we have used that $\L$ is a symmetric matrix.  If we let $\L = \Gamma \Lambda \Gamma^\top$ be an eigen-decomposition of $\L$ then we can also represent $\A(\widebar k)$ and $\A(\widebar k)^{-1/2}$ as:
\begin{align}
\A(\widebar k) &= \Gamma \left( (1+\gamma \widebar k^2) \I + \gamma \Lambda^2 + 2 \widebar k \gamma \Lambda \right) \Gamma^\top \\
 &= \Gamma (\I + \gamma ( \widebar k \I +  \Lambda)^2) \Gamma^\top \\ 
 \label{eq:A12}
\A(\widebar k)^{-1/2} &= \Gamma (\I + \gamma ( \widebar k \I +  \Lambda)^2)^{-1/2} \Gamma^\top 
\end{align}
Now we substitute back into the original polar problem:
\begin{align}
\Omega^\circ(\Z) &= \max_{\x, \widebar k} \| \A(\widebar k)^{-1/2} \Z \x \|_2 \ST \|\x\|_2^2 \leq 1 \\
\label{eq:k_line}
&= \max_{\widebar k} \|(\A(\widebar k)^{-1/2} \Z \|_2
\end{align}
where $\|\cdot\|_2$ denotes the spectral norm (maximum singular value).  Similarly, for a given $\widebar k$ the optimal $\x$ is given as the right singular vector of $\A(\widebar k)^{-1/2}\Z$ associated with the largest singular value.

As a result, we can solve the polar by performing a line search over $\widebar k$, then once an optimal $\widebar k^*$ is found we get $\x^*$ as the largest right singular vector of $\A(\widebar k^*)^{-1/2} \Z$ and the optimal $\d^*$ from \eqref{eq:d_opt} (where $\d_{opt}$ will be the largest left singular vector of $\A(\widebar k^*)^{-1/2} \Z$ multiplied by $\A(\widebar k^*)^{-1/2}$).

Now, an upper bound for $\widebar k$ can be calculated from the fact that the optimal $\widebar k$ is defined using a minimization problem, i.e.
\begin{eqnarray}
    \widebar k^* = \argmin_{\widebar k} \| \L \d + \widebar k \d \|^2
\end{eqnarray}
So we note for any $\d$,
\begin{eqnarray}
   \widebar k^* = -\frac{ \d^\top \L \d} {\|\d\|_2^2}
   \label{eq:opt_k}
\end{eqnarray}
which is bounded by the smallest eigenvalue of $\L$ (note that $\L$ is negative (semi)definite). We cite the literature on eigenvalues of discrete second derivatives \cite{chung2000discrete} to note that all eigenvalues of $\L$ (irrespective of the boundary conditions) lie in the range $[\frac{-4}{\Delta l},0]$, since we specifically chose $\Delta l = 1$ (without loss of generality), we have that all eigenvalues of $\L$ lie in the range$[-4,0]$.






As a result, we need to only consider $\widebar k$ in the range $[0,4]$:
\begin{equation}
\label{eqn:line_search_obj}
\widebar k^* = \argmax_{\widebar k \in [0,4]} \| \A(\widebar k)^{-1/2} \Z \|_2
\end{equation}

\end{proof}

\subsection{Optimization Algorithm is Solvable in Polynomial TIme}
\label{ap:finite_computations}

We first prove a lemma for Lipschitz constants.
\begin{lemma}
\label{lem:L_bound}
Given a set of constants $\lambda_i, \ i=1, 2, \ldots$ such that $\forall i$, $\mu_\lambda \leq \lambda_i \leq 0$, and a constant $\gamma > 0$, let the function $f$ be defined as:
\begin{equation}
f(x) = \max_i \frac{1}{\sqrt{1+\gamma(x + \lambda_i)^2}}.
\end{equation}
Then, over the domain $0 \leq x \leq \mu_x$ $f$ is Lipschitz continuous with Lipschitz constant $L_f$ bounded as follows:
\begin{equation}
\begin{split}
    L_f &\leq \begin{cases} \frac{2}{3\sqrt{3}} \sqrt{\gamma}, \qquad  \gamma\geq \frac{1}{2 \max \{ \mu_\lambda^2, \mu_x^2 \}} \\
\gamma \max \{ -\mu_\lambda, \mu_x\} (1+\gamma \max \{\mu_\lambda^2, \mu_x^2\} )^{-\tfrac{3}{2}}, \quad \mathrm{otherwise}
\end{cases}   \\ &\leq  \frac{2}{3 \sqrt{3}}\sqrt{\gamma}. 
\end{split}
\end{equation}
\end{lemma}

\begin{proof}
First, note that for any two Lipschitz continuous functions $\psi_a$ and $\psi_b$, with associated Lipschitz constants $L_a$ and $L_b$, respectively, one has that the point-wise maximum of the two functions, $\psi(x) = \max \{ \psi_a(x), \psi_b(x) \}$, is also Lipschitz continuous with Lipschitz constant bounded by $\max \{ L_a, L_b \}$.  This can be easily seen by the following two inequalities:
\begin{equation}
\begin{split}
    \psi_a(x') &\leq \psi_a(x) + | \psi_a(x') - \psi_a(x) | \\ &\leq \psi(x) + | \psi_a(x') - \psi_a(x) | \\ &\leq \psi(x) + L_a |x' - x| \\
\psi_b(x') &\leq \psi_b(x) + | \psi_b(x') - \psi_b(x) | \\ &\leq \psi(x) + | \psi_b(x') - \psi_b(x) | \\ &\leq \psi(x) + L_b |x' - x|
\end{split}
\end{equation}
From this we have:
\begin{equation}
\begin{split}
\psi(x') &= \max \{ \psi_a(x'), \psi_b(x') \} \\ &\leq \max \{ \psi(x) + L_a |x'-x|, \psi(x) + L_b |x'-x| \} \\ &= \psi(x) + \max \{L_a, L_b \} |x'-x| \\
&\implies \psi(x') - \psi(x) \leq \max \{ L_a, L_b \} |x'-x|
\end{split}
\end{equation}
which implies the claim from symmetry.

Now, if we define the functions $g$ and $h$ as
\begin{equation}
g(x, \lambda) = \frac{1}{\sqrt{1+\gamma(x + \lambda)^2}} \ \ \ \ \ \ \ h(b) = \gamma b (1+\gamma b^2)^{-\tfrac{3}{2}}
\end{equation}
we have that the Lipschitz constant of $f$, denoted as $L_f$, is bounded as:
\begin{equation}
\label{eq:h_max}
\begin{split}
L_f &\leq \max_i \sup_{x \in [0, u_x]} \left| \frac{\partial}{\partial x} g(x, \lambda_i) \right| \\ &\leq \sup_{\lambda \in [\mu_\lambda,0]} \sup_{x \in [0, \mu_x]} \left| \frac{\partial}{\partial x} g(x, \lambda) \right|\\ 
&= \sup_{\lambda \in [\mu_\lambda,0]} \sup_{x \in [0, \mu_x]} \left| - \frac{\gamma(x+\lambda)}{ (\sqrt{1+\gamma (x+\lambda)^2})^3 } \right| \\
&= \sup_{b \in [\mu_\lambda, \mu_x]} \left| h(b) \right| \\ &= \sup_{b \in [0, \max \{ -\mu_\lambda, \mu_x \} ]} h(b)
\end{split}
\end{equation}
Where the first inequality is from the result above about the Lipschitz constant of the point-wise maximum of two functions and the simple fact that the Lipschitz constant of a function is bounded by the maximum magnitude of its gradient, and the final equality is due to the symmetry of $|h(b)|$ about the origin.  Now, finding the critical points of $h(b)$ for non-negative $b$ we have:
\begin{equation}
\begin{split}
h'(b) &= \gamma (1+\gamma b^2)^{-\tfrac{3}{2}} - 3 \gamma^2 b^2 (1+\gamma b^2)^{- \tfrac{5}{2}} = 0 \\
\implies 3 \gamma b^2 &= (1+ \gamma b^2) \implies b^* = \frac{1}{\sqrt{2 \gamma}}
\end{split} 
\end{equation}
Note that $h(0)=0$, $h(b)>0$ for all $b>0$, and $h'(b)<0$ for all $b > b^*$.  As a result, $b^*$ will be a maximizer of $h(b)$ if it is feasible, otherwise the maximum will occur at the extreme point $b = \max \{-\mu_\lambda, \mu_x \}$.  From this we have the result:
\begin{equation}
\begin{split}
L_f  \\ &\leq \begin{cases} h(b^*) = \frac{2}{3 \sqrt{3}} \sqrt{\gamma} & \gamma\geq \frac{1}{2 \max \{ \mu_\lambda^2, \mu_x^2\}} \\
h(\max \{-\mu_\lambda, \mu_x \} )  & \text{otherwise}
\end{cases} \\
&= \begin{cases} \frac{2}{3 \sqrt{3}} \sqrt{\gamma} \qquad \qquad \quad \gamma\geq \frac{1}{2 \max \{ \mu_\lambda^2, \mu_x^2\}} \\
\gamma \max \{ -\mu_\lambda, \mu_x \} (1+\gamma \max \{ \mu_\lambda^2, \mu_x^2 \})^{-\tfrac{3}{2}} \quad \text{otherwise}
\end{cases} \\
& \leq h(b^*) = \frac{2}{3\sqrt{3}}\sqrt{\gamma}.
\end{split}
\end{equation}
\end{proof}

\begin{theorem}
\label{thm:Lip_line_search}
The line search in equation~\eqref{eqn:line_search_obj} over $\widebar k$ is Lipschitz continuous with a Lipschitz constant, $L_{\widebar k}$, which is bounded by:
\begin{equation}
L_{\widebar k} \leq \left[ \begin{cases} \frac{2}{3 \sqrt{3}} \sqrt{\gamma} \|\Z\|_2 & \gamma \geq \frac{1}{32} \\ 
4 \gamma (1 + 16 \gamma)^{-\tfrac{3}{2}} \|\Z\|_2 & \mathrm{otherwise} \end{cases} \right] \leq \frac{2}{3 \sqrt{3}} \sqrt{\gamma} \|\Z\|_2
\end{equation}
\end{theorem}

\begin{proof}

To show Lipschitz continuity, we define the function:
\begin{equation}
f_\A(\widebar k) = \| \A(\widebar k)^{-1/2} \|_2
\end{equation}
and then note the following:
\begin{equation}
\begin{split}
& \left| \|\A(\widebar k)^{-1/2}\Z \|_2 - \| \A(\widebar k ')^{-1/2} \Z \|_2 \right| \\
\leq & \left\| \left( \A(\widebar k)^{-1/2} - \A(\widebar k')^{-1/2} \right) \Z \right\|_2 \\
\leq & \left\| \A(\widebar k)^{-1/2} - \A(\widebar k')^{-1/2} \right\|_2 \left\| \Z \right\|_2 \\
\leq & L_\A |\widebar k - \widebar k'| \|\Z\|_2
\end{split}
\end{equation}
where the first inequality is simply the reverse triangle inequality, the second inequality is due to the spectral norm being submultiplicative, and $L_\A$ denotes the Lipschitz constant of $f_\A(\widebar k)$.  From the form of $\A(\widebar k)^{-1/2}$ in \eqref{eq:A12} note that we have:
\begin{equation}
f_\A(\widebar k) \equiv \| \A(\widebar k)^{-1/2}\|_2 = \max_i \frac{1}{\sqrt{1 + \gamma(\widebar k + \Lambda_{i,i})^2}}
\end{equation}
so the result is completed by recalling from our discussion above that $\Lambda_{i,i} \in [-4, 0], \forall i$ and applying Lemma \ref{lem:L_bound}.


\end{proof}

We also note that the above result implies that we can solve the polar by first performing a (one-dimensional) line search over $k$, and due to the fact that the largest singular value of a matrix is a Lipschitz continuous function, this line search can be solved efficiently by a variety of global optimization algorithms.  For example, we give the following corollary for the simple algorithm given in \cite{malherbe2017global}, and similar results can be obtained for other algorithms.

\begin{corollary}[Adapted from Cor 13 in \cite{malherbe2017global}]
\label{cor:line_search}
For the function $f(\bar k) = \|\A(\bar k)^{-1/2} \Z \|_2$ as defined in Theorem \ref{thm:polar}, if we let $\bar k_1, \ldots \bar k_r$ denote the iterates of the LIPO algorithm in \cite{malherbe2017global} then we have $\forall \delta \in (0,1)$ with probability at least $1-\delta$,
\begin{equation}
\max_{\bar k \in [0,4]} f(\bar k) - \max_{i=1\ldots r} f(\bar k_i) \leq \tfrac{8}{3 \sqrt{3}} \sqrt{\gamma} \|\Z\|_2 \frac{\ln(1/\delta)}{r}
\end{equation}
\end{corollary}

As a result, we have that the error of the linesearch converges with rate $\mathcal{O}(1/r)$, where $r$ is the number of function evaluations of $f(\bar k)$, to a global optimum, and then, given the optimal value of $k$, the optimal $(\d, \x)$ vectors can be computed in closed-form via singular value decomposition.  Taken together this allows us to employ the Meta-Algorithm defined in Algorithm \ref{alg:meta} to solve problem \eqref{eq:main_obj}.

Note that the above results show that the Lipschitz constant of the line search in the polar problem (and hence the efficiency of solving the polar problem) is bounded by a constant times a norm at the point at which we  calculate the polar $\Z = \nabla \frac{1}{2} \|Y - \D \X^\top \|_F^2 = \D \X^\top -\Y$.  Note that in Algorithm \ref{alg:meta} we evaluated the polar when $\D$ and $\X$ are first-order stationary points.  Here we briefly note that it is straight-forward to guarantee that the norm of the point at which we evaluate the polar (the residual of the least squared error term, $\D \X^\top - \Y$) is always globally bounded at first-order stationary points.    

To see this, consider a very simplified problem, where we hold the directions of $\D$ and $\X$ fixed in \eqref{eq:main_obj} and we simply optimize over a scalar multiplication of $\X \rightarrow \alpha \X$.  I.e., we want to solve:
\begin{equation}
\min_{\alpha} \tfrac{1}{2} \| \Y - \D (\alpha \X)^\top \|_F^2 + \tfrac{\lambda}{2} \| (\alpha \X) \|_F^2 + \tfrac{\lambda}{2} \sum_{i=1}^N \d_i^\top \A(k_i) \d_i
\end{equation}
It is easily shown that the optimal value for $\alpha$ in this case (assuming we are not in the trivial case where $\D$ or $\X$ are all-zeros) is given as:
\begin{equation}
\alpha^* = \frac{\langle \D \X^\top , \Y \rangle} {\| \D \X^\top \|_F^2 + \lambda \| \X \|_F^2}
\end{equation} 
Evaluating the residual for this optimal scaling of $\X$ gives:
\begin{equation}
\begin{split}
&\tfrac{1}{2} \|\Y - \D (\alpha^* \X)^\top \|_F^2 \\ =
&\tfrac{1}{2}\||\Y\|_F^2 + \frac{ ( \langle \Y, \D \X^\top \rangle )^2} {\| \D \X^\top \|_F^2 + \lambda \|\X\|_F^2}  \left( \frac{ \|\D \X^\top \|_F^2} {\| \D \X^\top \|_F^2 + \lambda \| \X\|_F^2} - 1 \right) \\
&\leq \tfrac{1}{2} \| \Y \|_F^2
\end{split}
\end{equation}
where the final inequality can be seen by noting the term in the final parenthesis is always non-positive and the leading coefficient is non-negative.  Note that any first-order stationary point will be at an optimal scaling of $\alpha$ since otherwise we can improve the objective by simply scaling $\X$, which would imply we are not at a first order stationary point.  As a result, we have that the residual term when we evaluate it in Algorithm \ref{alg:meta} will always be bounded at least by $\tfrac{1}{2} \|\Y\|_F^2$, which also trivially implies a bound on the spectral norm of the residual since we always have $\|\Z\|_2 \leq \| \Z\|_F$, $\forall \Z$.
{

\subsection{Proof of Optimal Step Size}
\label{appendix:optimal_step_size}
\begin{proof}
Let $f(\tau)$ represent the objective function in \eqref{eqn:for_tau} with everything except for $\tau$ held fixed:
\begin{equation}
    \begin{split}
        f(\tau) &= \tfrac{1}{2}\|\Y- \D_\tau \X_\tau^\top\|_F^2 + \tfrac{\lambda}{2} \sum_{i=1}^{N+1} (\d_{\tau})_i^{\top} \A_\tau(k_i) (\d_{\tau})_i
    \end{split}
\end{equation}
Observe that from the solution to the Polar problem we have by construction that the regularization term $\theta(\d^*,\x^*)=1$, so combined with the positive homogeneity of $\theta$ we have have that minimizing $f(\tau)$ w.r.t. $\tau$ is equivalent to solving:
%
\begin{equation}
\min_{\tau \geq 0} \tfrac{1}{2}\|\Y- \widetilde \D \widetilde \X^\top - \tau^2 \d^* (\x^*)^\top \|_F^2 + \lambda \tau^2
\end{equation}
Taking the gradient w.r.t. $\tau^2$ and solving for 0 gives:
\begin{equation}
(\tau^*)^2 = \frac{\langle \Y-\widetilde \D \widetilde \X^\top, \d^* (\x^*)^\top \rangle - \lambda}{ \|\d^* (\x^*)^\top\|_F^2}
\end{equation}
The result is completed by noting that the numerator is guaranteed to be strictly positive due to the fact that the Polar solution has value strictly greater that 1.
\end{proof}